\documentclass[10pt,journal,compsoc]{IEEEtran}

\makeatletter
\def\endthebibliography{%
	\def\@noitemerr{\@latex@warning{Empty `thebibliography' environment}}%
	\endlist
}
\makeatother

\ifCLASSOPTIONcompsoc
  \usepackage[nocompress]{cite}
\else
  \usepackage{cite}
\fi

\usepackage{hyperref}       %
\hypersetup{colorlinks=true,urlcolor=black}
\usepackage{url}            %
\usepackage{booktabs}       %
\usepackage{amsfonts}       %
\usepackage{nicefrac}       %
\usepackage{microtype}      %
\usepackage{algorithm}
\usepackage{algorithmic}
\usepackage{graphicx}
\usepackage{epsfig}
\usepackage{amsmath}
\usepackage{amssymb}
\usepackage{amsthm}
\usepackage{subfigure}
\usepackage{booktabs} %
\usepackage{multirow}
\usepackage{etoolbox}
\usepackage{colortbl}
\usepackage{footnote}
\usepackage{arydshln}
\usepackage{enumitem}
\usepackage{ragged2e}
\usepackage{threeparttable}
\usepackage{bbm}
\usepackage{xspace}
\usepackage{xcolor}
\usepackage[export]{adjustbox}

\usepackage{amsmath,amsfonts,bm}

\def\eqref#1{equation~\ref{#1}}

\def\1{\bm{1}}

\def\vs{{\bm{s}}}

\def\mM{{\bm{M}}}

\DeclareMathAlphabet{\mathsfit}{\encodingdefault}{\sfdefault}{m}{sl}
\SetMathAlphabet{\mathsfit}{bold}{\encodingdefault}{\sfdefault}{bx}{n}

\def\gD{{\mathcal{D}}}

\def\gN{{\mathcal{N}}}

\def\sE{{\mathbb{E}}}

\def\sR{{\mathbb{R}}}

\def\sZ{{\mathbb{Z}}}

\def\bK{{\bf K}}

\def\bW{{\bf W}}

\def\bz{{\bf z}}
\def\bx{{\bf x}}

\def\br{{\bf r}}
\def\bw{{\bf w}}
\def\eg{\emph{e.g.}} 

\def\ie{\emph{i.e.}} 

\def\etc{\emph{etc.}} 

\newtheorem{coll}{Corollary}

\newtheorem{thm}{Theorem}
\newtheorem{prop}{Proposition}
\newtheorem{lemma}{Lemma}

\def\vs{\emph{v.s.}}

\definecolor{citecolor}{HTML}{0071bc}
\usepackage{hyperref}
\hypersetup{
  colorlinks,
  citecolor=citecolor,
  linkcolor=red
}

\newcommand{\tabincell}[2]{\begin{tabular}{@{}#1@{}}#2\end{tabular}}

\definecolor{mypink}{cmyk}{0, 0.7808, 0.4429, 0.1412}

\def\jing{\textcolor{black}}
\def\jingr{\textcolor{black}}
\def\jingrs{\textcolor{black}}
\def\revise{\textcolor{black}}
\def\liu{\textcolor{black}}
\def\chenp{\textcolor{black}}
\def\guo{\textcolor{black}}
\def\hk{\textcolor{black}}

\definecolor{seagreen}{RGB}{62,187,163}
\definecolor{ZhuangGreen}{HTML}{d4fcd8}
\def\zheng{\textcolor{black}}

\newcommand{\methodfullname}{Single-path Bit Sharing\xspace}
\newcommand{\methodshortname}{SBS\xspace}
\def\mytitle{Single-path Bit Sharing for Automatic Loss-aware Model Compression}

\begin{document}

\title{\mytitle}

\author{
	Jing Liu, Bohan Zhuang, Peng Chen, Chunhua Shen, Jianfei Cai, Mingkui Tan$^\dagger$
	\IEEEcompsocitemizethanks{
	    \IEEEcompsocthanksitem Jing Liu is with the School of Software Engineering, South China University of Technology and also with the Key Laboratory of Big Data and Intelligent Robot (South China University of Technology), Ministry of Education. E-mail: seliujing@mail.scut.edu.cn.
		\IEEEcompsocthanksitem Mingkui Tan is with the School of Software Engineering, South China University of Technology. Mingkui Tan is also with the Pazhou Laboratory, Guangzhou, China. E-mail: mingkuitan@scut.edu.cn.
		\IEEEcompsocthanksitem  Peng Chen is with the School of Computer Science, the University of Adelaide, Australia. E-mail: blueardour@gmail.com.
		\IEEEcompsocthanksitem Bohan Zhuang, and Jianfei Cai are with the Faculty of Information Technology, Monash University, Australia. E-mail: \{bohan.zhuang, jianfei.cai\}@monash.edu.
		\IEEEcompsocthanksitem Chunhua Shen is with the College of Computer Science and Technology, Zhejiang University, Hangzhou, China. E-mail: chhshen@gmail.com.
        \IEEEcompsocthanksitem $^\dagger$ Corresponding author.
	}
}

\IEEEtitleabstractindextext{%

\begin{abstract}
\justifying
Network pruning and quantization are proven to be effective ways for deep model compression. 
To obtain a highly compact model, most methods first perform network pruning and then conduct  
 quantization based on the pruned model. 
However, this strategy may ignore that the pruning and quantization would affect each other and thus performing them separately may lead to sub-optimal performance. 
To address this, performing pruning and quantization jointly is essential. 
Nevertheless, how to make a trade-off between pruning and quantization is non-trivial.
Moreover, existing compression methods often rely on some pre-defined compression configurations (\ie, pruning rates or bitwidths). 
Some attempts have been made to search for optimal configurations, which however may take unbearable optimization cost.
\jingr{To address these issues, we devise a simple yet effective method named \methodfullname (\methodshortname) for automatic loss-aware model compression. 
To this end, we consider the network pruning as a special case of quantization and provide a unified view for model pruning and quantization.
We then introduce a single-path model to encode all candidate compression configurations, 
where a high bitwidth value will be decomposed into the sum of a lowest bitwidth value and a series of re-assignment offsets. Relying on the single-path model, we introduce learnable binary gates to encode the choice of configurations and learn the binary gates and model parameters jointly.
More importantly, the configuration search problem can be transformed into a subset selection problem, which helps to significantly reduce the optimization difficulty and computation cost.
In this way, the compression configurations of each layer and the trade-off between pruning and quantization can be automatically determined.
}
Extensive experiments on CIFAR-100 and ImageNet show that \methodshortname significantly reduces computation cost while achieving promising performance. 
\jing{For example, our \methodshortname compressed MobileNetV2 achieves 22.6$\times$ Bit-Operation (BOP) reduction with only 0.1\% drop in the Top-1 accuracy.}
\end{abstract}

}

\maketitle

\IEEEdisplaynontitleabstractindextext

\IEEEpeerreviewmaketitle

\section{Introduction}
Deep neural networks (DNNs)~\cite{lecun1989backpropagation} have achieved great success in many challenging computer vision tasks, including image classification~\cite{krizhevsky2012imagenet,he2016deep,gao2021res2net}, object detection~\cite{ren2016faster,lin2017focal,tian2019fcos}, image generation~\cite{goodfellow2014generative,guo2019auto,cao2020lcc}, and video analysis~\cite{simonyan2014two,wang2016temporal,zeng2019breaking}.
\jing{However, a deep model usually has a large number of parameters and consumes enormous computational resources, which presents great obstacles for many
applications, especially on resource-constraint devices, such as smartphones.}
To reduce the number of parameters and computational overhead, many methods~\cite{he2019filter,zhou2016dorefa,zhuang2018discrimination} have been proposed to perform model compression by removing the redundancy while maintaining the performance.

\begin{figure*}[t]
    \centering
    \subfigure[Multi-path search \jing{scheme~\cite{wu2018mixed}}.]{
        \includegraphics[width=0.42\linewidth]{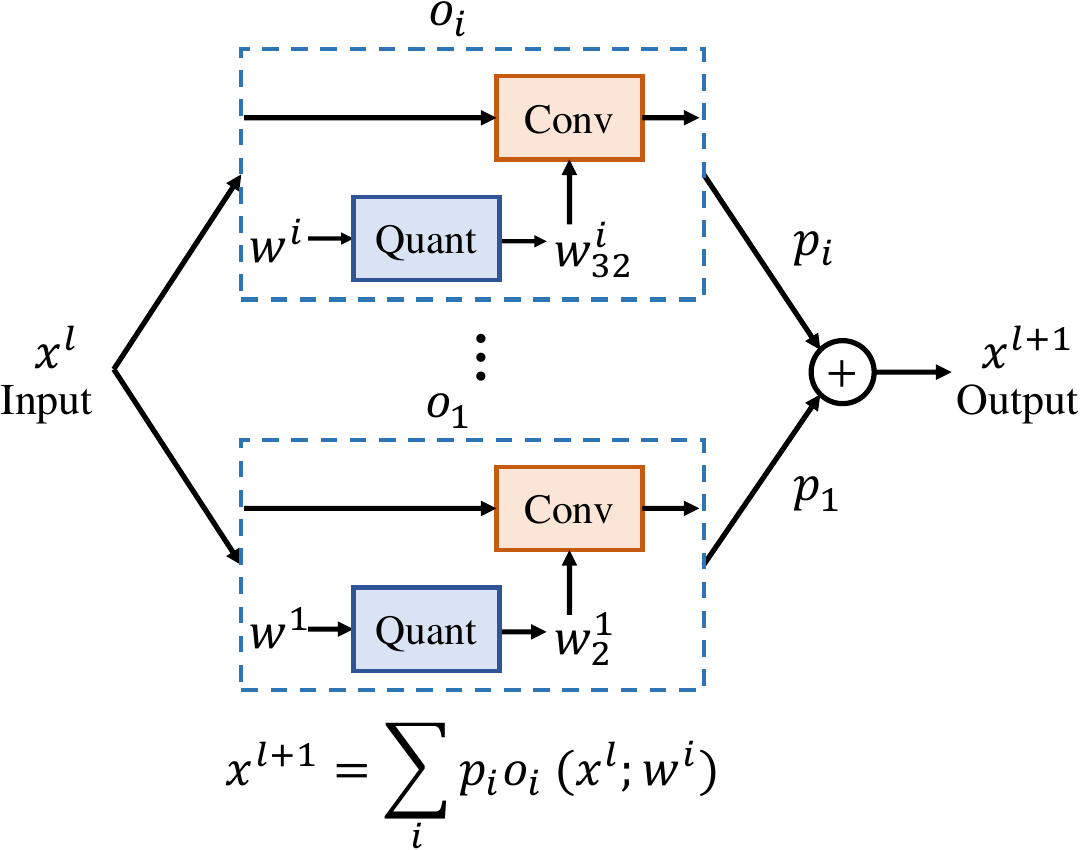}
        \label{fig:diagram_multi_path}
    }
    \subfigure[\jing{Single-path search \jing{scheme} (Ours).}]{
        \includegraphics[width=0.52\linewidth]{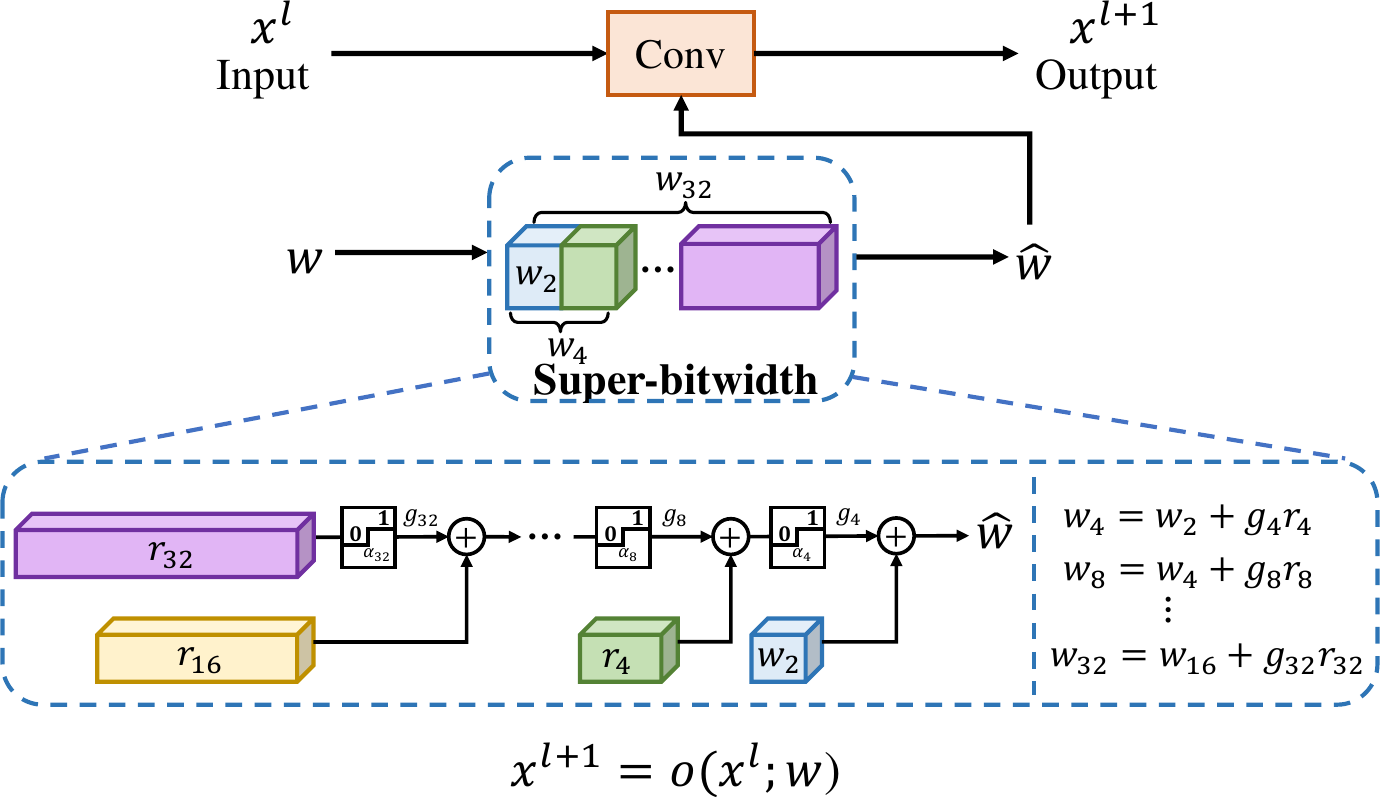}
        \label{fig:diagram_signle_path}
    }
    \caption{Multi-path \vs~single-path search scheme. 
    \jing{For clarity, we show the illustrations of weight quantization only.
    }
    (a) \liu{Multi-path search scheme~\cite{wu2018mixed}: \jingr{it} represents each candidate configuration as a separate path and formulates the mixed-precision quantization problem as a path selection problem, which gives rise to a huge number of trainable parameters and high computational overhead when the search space becomes large}. $o_i$ and $p_i$ are the convolutional operation and probability for the $i$-th configuration, respectively. $x^l$ indicates the input feature maps for layer $l$, $w^i$ is the weight for the $i$-th configuration and $w^i_k$ is the $k$-bit quantized version of $w^i$. 
    (b) Single-path search scheme (Ours): \jingr{\jingr{it} only maintains a single path model and represents each candidate configuration as a subset of a super-bitwidth. Therefore, the mixed-precision quantization problem can be formulated as a subset selection problem,} which significantly reduces the number of parameters, computation cost, and optimization difficulty. The super-bitwidth is the highest bitwidth (\ie, 32-bit) in the search space, $w$ is the weight for the convolutional operation, $w_k$ is the $k$-bit quantized version of $w$, $r_k$ is the re-assignment offset and $g_k$ is the corresponding binary gate to control the bitwidth decision.
    }
    \label{fig:mutl_path_vs_single_path}
\end{figure*}

\jing{In recent years, we have witnessed the remarkable progress of model compression methods.} Specifically, network pruning~\cite{he2017channel, he2019filter} removes the uninformative modules and network quantization~\cite{zhou2016dorefa, hubara2016binarized} 
maps the full-precision values to low-precision ones. To obtain a highly compact model, \jing{most methods~\cite{han2016deep,LouizosUW17,ye2018unified} first perform  pruning and then conduct quantization based on the pruned model.}
\jingr{However, performing pruning and quantization {separately may lead to suboptimal performance as it ignores that the two procedures often affect each other.}} 
\jingr{Therefore, it is essential and urgent to perform pruning and quantization jointly in practical applications, which, however, is nontrivial and may incur some new challenges.}

First, finding the optimal trade-off between pruning and quantization is non-trivial \jingrs{as they would affect each other.}
For example, if a model is under-pruned, we can apply aggressive network quantization to the pruned model to achieve a high compression ratio.
\jing{In contrast, if a model is over-pruned, the resulting model is more sensitive to the quantization noise. In this case, the quantization bitwidth of the pruned model has to be high to preserve the performance.}

Second, 
to achieve better performance, one may assign different configurations \jingr{(\ie, pruning rates and bitwidths)} according to the layer’s contribution to the accuracy and efﬁciency of a network. 
However, the search space of model compression grows exponentially with the increasing number of compression configurations. 
To handle this, 
existing differentiable methods~\cite{wu2018mixed,cai2020rethinking} 
explore the exponential search space using gradient-based optimization.
As shown in Fig.~\ref{fig:diagram_multi_path}, these methods first construct a multi-path \jing{network} \revise{where each path denotes a candidate compression configuration.}
Then, the optimal path is selected during training. However, 
when the search space becomes large, the multi-path scheme \jingr{will lead to} a huge amount of parameters and \jingr{incur high computation cost}. \jingr{As a result, the optimization of the compressed network could be very challenging due to the parameter explosion.}

In this paper, \revise{we propose to address the above challenges by devising a simple yet effective method named \methodfullname (\methodshortname) for automatic loss-aware model compression. }
\jingr{To this end, we provide a unified view for the pruning and quantization, where pruning is reformulated as a special case of quantization. We then introduce a novel single-path model to encode all bitwidths in the search space. As shown in Fig.~\ref{fig:diagram_signle_path}, 
\revise{we theoretically show that}
the quantized values of a high bitwidth can be shared with those of low bitwidths under some conditions. Therefore, we are able to decompose a quantized representation into the sum of the lowest bitwidth representation and a series of re-assignment offsets. In this case, the quantized values of different bitwidths are different subsets of those of super-bitwidth (\ie, the highest bitwidth). 
Compared with the multi-path scheme~\cite{wu2018mixed}, 
we only need to maintain a single path model, which requires much fewer model parameters. \revise{More importantly, we are able to transform 
the mixed-precision quantization problem 
into a subset selection problem, which significantly reduces the computation cost and alleviates the optimization difficulty.}
Relying on the single-path model, \revise{we further introduce learnable binary gates} to encode the choice of bitwidth and learn the binary gates and network parameters jointly. 
In this way, the configurations of each layer can be automatically determined and the trade-off between pruning and quantization can be optimized.
}

Our main contributions are summarized as follows:
\begin{itemize}
\item We devise a novel single-path scheme that \jing{encapsulates multiple configurations} in a unified single-path framework, \jingr{which requires fewer model parameters compared with multi-path scheme.}

\item \jingr{
We transform the model compression into a subset selection problem and explicitly share the quantized values among various bitwidths. As a result, \methodshortname enables the candidate configurations to learn jointly rather than separately and thus significantly reduces the computation cost and alleviates the optimization difficulty.
}

\item
We formulate the quantized representation as a gated combination of the lowest bitwidth representation and a series of re-assignment offsets.
\jingr{By training the binary gates and network parameters, the configuration of each layer and the trade-off between pruning and quantization can be automatically determined. }

\item
We evaluate our \methodshortname on CIFAR-100 and ImageNet over various network architectures. Extensive experiments show that the proposed method achieves the state-of-the-art performance. For example, on ImageNet, our \methodshortname compressed MobileNetV2 achieves 22.6$\times$ Bit-Operation (BOP) reduction with only 0.1\% performance decline in terms of the Top-1 accuracy.
\end{itemize}

\section{Related Work}
\label{sec:related_work}

\noindent\textbf{Network quantization.}
Network quantization represents the weights, activations and even gradients with low precision to yield compact DNNs. With low-precision integers~\cite{zhou2016dorefa} or power-of-two representations~\cite{Li2020Additive}, the heavy matrix multiplications can be replaced by efficient bitwise operations, leading to much faster test-time inference and lower \jingr{resource consumption. \revise{According to the quantization bitwidth, existing quantization methods can be roughly categorized into two categories, namely, fixed-point quantization~\cite{zhou2016dorefa,zhang2018lq,jung2019learning,Esser2020LEARNED,kim2021bert,liu2021posttraining,shen2020q,han2021improving} and binary quantization~\cite{hubara2016binarized,rastegari2016xnor,Zhuang_2019_CVPR,liu2018bi,qin2020forward,qin2021bibert,liu2021adam}.}}
To improve the quantization performance, extensive methods have been proposed to learn accurate quantizers~\cite{jung2019learning,zhang2018lq,choi2018pact,cai2017deep,Esser2020LEARNED,Bhalgat_2020_CVPR_Workshops}. Specifically, given a convolutional layer, let $x$  and $w$ be the output activations of the previous layer and the weight parameters of given layer, respectively. First, following~\cite{jung2019learning,choi2018pact,chen2020aqd}, one can normalize $x$  and $w$ into scale [0, 1] by $T_x(\cdot)$ and $T_w(\cdot)$, respectively:
\begin{align}
    z_x &= T_x(x) = \mathrm{clip}\left(\frac{x}{v_x}, 0, 1\right),  \\
    z_w &= T_w(w) = \frac{1}{2} \left( \mathrm{clip}\left(\frac{w}{v_w}, -1, 1\right) + 1 \right),
\end{align}
where $v_x$ and $v_w$ are trainable quantization intervals indicating the range of weights and activations to be quantized. \jing{Here, the function $\mathrm{clip}\left(v, v_\mathrm{low}, v_\mathrm{up}\right) = {\rm min}({\rm max}(v,v_\mathrm{low}),v_\mathrm{up})$ clips any number $v$ into the range $[v_\mathrm{low}, v_\mathrm{up}]$}. Then, one can apply the following function to quantize the normalized activations and parameters, namely  $z_x \in [0, 1]$ and $z_w  \in [0, 1]$, to discretized ones:
\begin{equation}
    \label{eq:discretizer}
    D(z, s) = s \cdot \mathrm{round}\left(\frac{z}{s}\right),
\end{equation}
where $s$ denotes the normalized step size, \jing{
$\mathrm{round}(x)=\lceil x-0.5\rceil$ returns the nearest integer of a given value $x$ and $\lceil \cdot \rceil$ is the ceiling function.} Typically, for $k$-bit quantization, the normalized step size $s$ can be computed by
\begin{equation}
    \label{eq:step_size}
    s = \frac{1}{2^k - 1}.
\end{equation}
Last, the quantized activations and weights can be obtained by $Q_x(x) = T_x^{-1}(D(z_x,s)) = v_x \cdot D(z_x,s)$ and $Q_w(w) = T_w^{-1}(D(z_w,s)) = v_w \cdot (2 \cdot D(z_w,s) -  1)$,
where $T_x^{-1}(\cdot)$ and $T_w^{-1}(\cdot)$ denote the inverse functions of $T_x(\cdot)$ and $T_w(\cdot)$, respectively. 
\jing{
In general, the function $D(\cdot, \cdot)$ is non-differentiable.  Following~\cite{zhou2016dorefa,hubara2016binarized}, one can use the straight through estimation (STE)~\cite{bengio2013estimating} to approximate the gradient of $D(\cdot, \cdot)$ by the identity mapping, namely, $\partial D (z, s)/ \partial z \approx 1$.} 

To reduce the optimization difficulty introduced by non-differentiable discretization, several methods have been proposed to approximate the gradients of $D(\cdot, \cdot)$~\cite{ding2019regularizing, louizos2019relaxed, Zhuang_2020_CVPR}.
Moreover, \revise{most previous works assign the same bitwidth for all layers~\cite{zhou2016dorefa,zhuang2018towards,Zhuang_2019_CVPR,jung2019learning,jin2019towards,Li2020Additive,Esser2020LEARNED,qin2022distribution,xu2020generative,chen2020aqd}.
\chenp{Though} attractive for simplicity, setting a uniform precision places no guarantee on optimizing network performance, since different layers have different redundancy and arithmetic intensity.}
Therefore, \jingr{several studies 
\chenp{proposed} 
mixed-precision quantization~\cite{wang2019haq,dong2019hawq,wu2018mixed,Uhlich2020Mixed,yao2021hawq,dong2019hawqv2,cai2020rethinking,chen2021towards,wang2021generalizable} that assigns} different bitwidths according to the redundancy of each layer. 
\jing{In this paper, based on the proposed single-path bit sharing model, we devise an approach that efficiently searches for appropriate bitwidths for different layers through gradient-based optimization. \revise{Apart from quantization, our \methodshortname also conducts pruning and automatically learns the trade-off between them, which often results in compact models with better performance.}}

\noindent\textbf{\jingr{Neural architecture search (NAS) and pruning.}}
\jingr{NAS} aims to automatically design efficient architectures.  \jingr{According to the search algorithm, existing methods either based on reinforcement learning~\cite{NASRL2017,pham2018efficient,yong2019nat}, evolutionary search~\cite{real2019regularized,nsganetv22020,yang2020cars} or gradient-based methods~\cite{liu2018darts,cai2018proxylessnas,wu2019fbnet}. In particular, gradient-based NAS has gained increased popularity, where the search space is relaxed to be continuous, allowing efficient architecture search using gradient descent. Depending on whether each operation can be added via a separate path or not, the search space can be categorized into multi-path design~\cite{liu2018darts,cai2018proxylessnas} and single-path formulation~\cite{stamoulis2019single,stamoulis2020single}.} 

\jing{While prevailing methods optimize the network topology, we focus on searching optimal pruning and quantization configurations for a given network.} Moreover, network pruning can be treated as \jingr{fine-grained NAS~\cite{Guo_2020_CVPR,Liu_2019_ICCV,dong2019network}} which removes redundant modules to reduce the model size and accelerate the run-time inference speed, \jing{giving rise to methods based on \jingr{weight pruning~\cite{han2016deep, guo2016dynamic,liu2018rethinking,han2015learning}, 
 filter pruning~\cite{he2017channel,zhuang2018discrimination,luo2019thinet,he2019filter},
or layer pruning~\cite{NIPS2016_41bfd20a,huang2018data,chen2019shallowing,Lin_2019_CVPR}, \etc~}Apart from filter pruning, we also perform network quantization to obtain more compact networks.}

\noindent\textbf{AutoML for model compression.}
\jing{
Recently, much effort has been devoted to automatically determining the \jingr{pruning rate~\cite{tung2018clip,dong2019network,liu2019metapruning} or the 
bitwidth~\cite{lou2019autoqb,cai2020rethinking,wu2018mixed} of each layer, either based on reinforcement learning~\cite{he2018amc,wang2019haq,lou2019autoqb}, evolutionary search~\cite{liu2019metapruning,wang2020apq}, gradient optimization~\cite{wu2018mixed,cai2020rethinking,gao2021network}, \etc~}To increase the compression ratio, several methods have been proposed to jointly optimize pruning and quantization strategies. \jingr{In particular, some work only support weight quantization~\cite{tung2018clip,ye2018unified,hu2021opq} or use fine-grained pruning~\cite{yang2020automatic,hu2021opq}}. However, the resultant models cannot be implemented efficiently on edge devices. To handle this, several methods~\cite{wu2018mixed,wang2020apq,wang2020diff} have been proposed to consider filter pruning, weight quantization, and activation quantization jointly. Compared with these methods, we carefully design the compression search space by 
\zheng{sharing the quantized values between different candidate configurations,}
which significantly reduces the number of parameters, search cost, and optimization difficulty. 
Compared with those methods that share the similarities of using quantized residual errors~\cite{chen2010approximate,gong2014compressing,li2017performance,van2020bayesian}, our proposed method recursively uses quantized residual errors to decompose a quantized representation into a set of candidate bitwidths and parameterize the bitwidth selection via a series of binary gates.} 
\jing{
Compared with Bayesian Bits~\cite{van2020bayesian}, our \methodshortname differs in several aspects: 1) We theoretically verify the theorem of quantization decomposition and it can be applied to non-power-of-two bitwidths (See Section~\ref{sec:comparison_power_of_two_non_power_of_two}),
which is a general case of the one in Bayesian Bits. 2) The optimization problems are different. Specifically, we formulate model compression as a single-path subset selection problem while Bayesian Bits casts the optimization of the binary gates into a variational inference problem that requires more relaxations and hyperparameters. 3) Our compressed models with less or comparable BOPs outperform those of Bayesian Bits by a large margin on ImageNet (See Table~\ref{table:results_on_imagenet}).}

\section{Proposed Method}
\jing{
Given a pre-trained model, we focus on 
automatic model compression for pruning and quantization jointly, which poses two new challenges.
First, finding the optimal trade-off between pruning and quantization is non-trivial since they may affect each other. Second, to determine the optimal configurations (\eg, pruning rates and bitwidths) for each layer, one may consider different configurations as different paths and reformulate the configuration search problem as a path selection problem~\cite{wu2018mixed}, as shown in Fig.~\ref{fig:diagram_multi_path}. However, when the search space becomes large, it suffers from a huge number of parameters and high computation cost. \jingr{Moreover, different candidate configurations are trained separately and thus the optimization of the compressed model \jingr{may become} more challenging.}
}

\jingr{In this paper, we first provide a unified view for model compression so that
we can perform pruning and quantization jointly. 
Specifically, we consider network pruning as a special case of quantization and formulate the joint pruning and quantization problem as a mixed-precision quantization problem. 
We then propose a \methodfullname (\methodshortname) scheme to encode all candidate configurations into a single path model, as shown in Fig.~\ref{fig:diagram_signle_path}. It is worth mentioning that our \methodshortname has much fewer model parameters than the most related work~\cite{wu2018mixed} and thus requires less computation cost.
Moreover, we are able to cast the mixed-precision problem into a subset selection problem and alleviate the optimization difficulties.}
\jing{
In the following subsections, we will illustrate each component of our method.
}

\vspace{-0.05in}
\subsection{Single-path bit sharing decomposition}
\label{sec:problem_definition}
To illustrate the single-path bit sharing decomposition, we begin with an example of 2-bit quantization for $z \in \{ z_x, z_w \}$.
Specifically, we use the following equation to quantize $z$ to $2$-bit using Eqs.~(\ref{eq:discretizer}) and~(\ref{eq:step_size}):
\begin{equation}
    \label{eq:2_bit_quantization}
    z_2 = D(z, s_2), \quad ~s_2=\frac{1}{2^2-1},
\end{equation}
where $z_2$ and $s_2$ are the quantized value and the step size of 2-bit quantization, respectively. 
Due to the large step size, the residual error $z - z_2 \in [-s_2/2, s_2/2]$ may be large and results 
in a significant performance decline. \jing{To reduce the residual error, an intuitive way is to use a smaller step size, which indicates that we quantize $z$ to a higher bitwidth.} Since the step size $s_4 = 1/(2^4-1)$ in 4-bit quantization is a divisor of the step size $s_2$ in 2-bit quantization, the quantized values of 2-bit quantization are shared with those of 4-bit quantization.
\jing{Based on 2-bit quantization, the 4-bit counterpart 
introduces}
additional unshared quantized values. In particular, if $z_2$ has zero residual error, then 4-bit quantization maps $z$ to the shared quantized values (\ie, $z_2$). \jing{In contrast, if $z_2$ has non-zero residual error, 4-bit quantization is likely to map $z$ to the unshared quantized values.}
In this case, 4-bit quantization can be regarded as performing quantized value re-assignment based on $z_2$. Such a re-assignment process can be formulated as
\begin{equation}
    \label{eq:4_bit_decomposition}
    z_4 = z_2 + r_4,
\end{equation}
where $z_4$ is the 4-bit quantized value and $r_4$ is the re-assignment offset based on $z_2$. To ensure that the results of re-assignment fall into the 4-bit quantized values, the re-assignment offset $r_4$ must be an integer multiple of the step size $s_4$.
Formally, $r_4$ can be computed by performing 4-bit quantization on the residual error of $z_2$:
\begin{equation}
    \label{eq:2_bit_residual_quantization}
    r_4 = D(z - z_2, s_4), \quad ~s_4 = \frac{1}{2^4 - 1}.
\end{equation}
According to Eq.~(\ref{eq:4_bit_decomposition}), a 4-bit quantized value can be decomposed into the 2-bit representation and its re-assignment offset. Similarly, an 8-bit quantized value can also be decomposed into the 4-bit representation and its corresponding re-assignment offset. In this way, we can generalize the idea of decomposition to arbitrary effective bitwidths as follows.
\begin{thm}
\textbf{\emph{(Quantization Decomposition)}} 
\label{prop:bit_decompistion}
    Let $z \in \left[ 0, 1 \right]$ be a normalized full-precision input, and $\{b_j\}_{j=1}^{K} $ be a sequence of candidate bitwidths. If $b_j$ is an integer multiple of $b_{j-1}$, \ie, $b_{j} \small{=} \gamma_j b_{j-1}(j>1)$, where $\gamma_j\in \sZ^+ \backslash \{1\}$ is a multiplier, then the quantized approximation $z_{b_K}$ can be decomposed as: 
        \begin{equation}
        \begin{aligned}
        \label{eq:bit_decompistion}
        &z_{b_K} = z_{b_1} + \sum_{j=2}^{K}r_{b_j}, ~\\ &\mathrm{where}~r_{b_{j}} = D(z - z_{b_{j-1}}, s_{b_{j}}), ~\\
        &~~~~~~~~~~~z_{b_j} = D(z , s_{b_j}), ~\\
        &~~~~~~~~~~~s_{b_j} =\frac{1}{2^{b_j} - 1}.
        \end{aligned}
        \end{equation}
\end{thm}
\jingr{{From Theorem \ref{prop:bit_decompistion},} the quantized representation $z_{b_K}$ can be decomposed into the sum of the lowest bitwidth representation $z_{b_1}$ and a series of recursive re-assignment offsets.}
\jingr{In this way, we can gradually reduce the error brought from  quantization by introducing the recursive re-assignment offsets.}
\jing{To measure the quantization error,}
we introduce the following corollary.

\begin{figure}[!t]
	\centering
	\includegraphics[width = 0.61\linewidth]{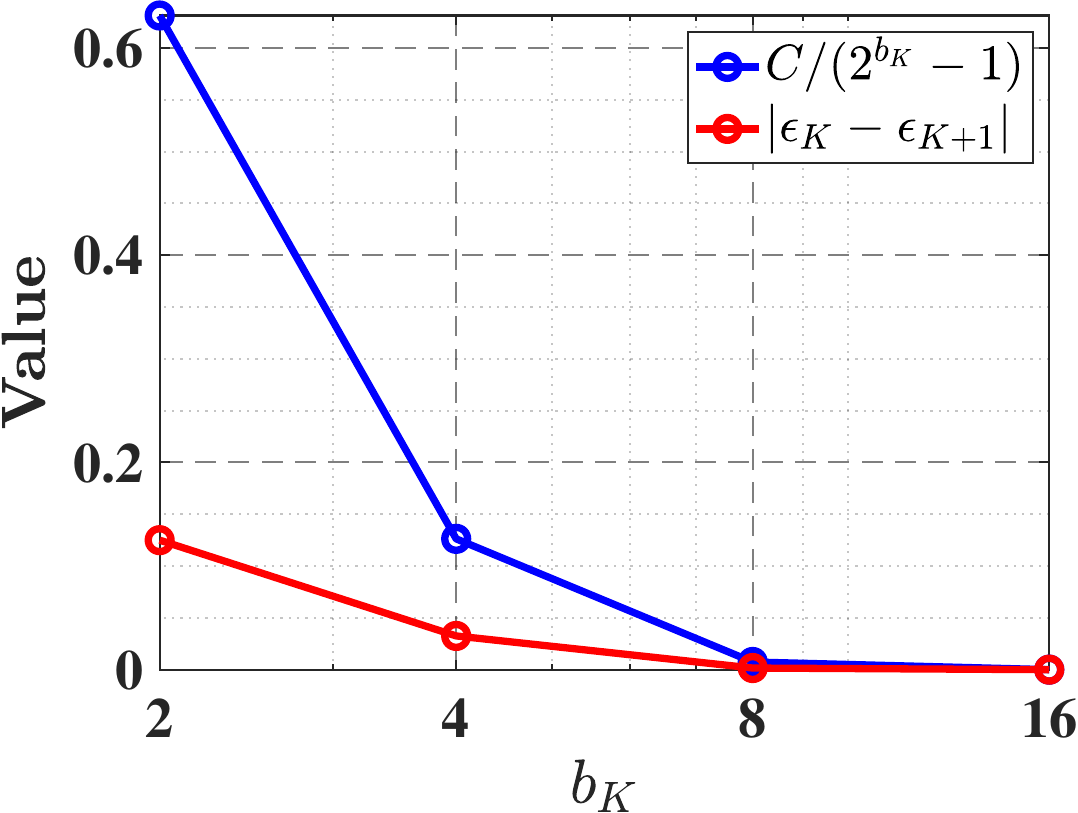}
	\caption{\jing{The normalized quantization error change and error bound \vs~bitwidth. The normalized quantization error change (red line) is bounded by $C/(2^{b_{K}}\small{-}1)$ (blue line).}
	}
	\label{fig:quantization_error}
\end{figure}

\begin{coll}
\textbf{\emph{(Normalized  Quantization Error Bound)}} 
\label{coll:bit_error}
Given $\bz \in [0, 1]^d$ being a normalized full-precision vector, $\bz_{b_K} $ being its quantized vector with bitwidth $b_K$, where $d$ is the cardinality of $\bz$. Let $\epsilon_K = \frac{\|\bz - \bz_{b_K}\|_1}{\|\bz\|_1}$ be the normalized quantization error, then the following error bound w.r.t. $K$ holds:
\begin{equation}
    \begin{aligned}
    |\epsilon_K \small{-} \epsilon_{K+1}|\leq \frac{C}{2^{b_{K}}\small{-}1},
   \end{aligned}
\end{equation}
where $~C \small{=} \frac{d}{\|\bz\|_1}$ is a constant.
\end{coll}
\jing{
To empirically demonstrate Corollary~\ref{coll:bit_error}, we perform quantization on a random toy data $\bz \in [0, 1]^{100}$ and show the normalized quantization error change $|\epsilon_K \small{-} \epsilon_{K+1}|$ and error bound $\frac{C}{2^{b_{K}}\small{-}1}$ in Fig.~\ref{fig:quantization_error}. From the results, the normalized quantization error change decreases quickly as the bitwidth increases and is bounded by $C/(2^{b_{K}}\small{-}1)$. 
We put the proof of Corollary~\ref{coll:bit_error} in the supplementary material.
}

\jing{Note that in Theorem~\ref{prop:bit_decompistion}, both the smallest bitwidth $b_1$ and the multiplier $\gamma_j$ can be set to arbitrary appropriate integer values (\eg, 2, 3, \etc).
To obtain a hardware-friendly compressed network\footnote{More details can be found in the supplementary material.}, we set $b_1$ and $\gamma_j$ to 2, which ensures that all the decomposition bitwidths are power-of-two. 
Moreover, since the normalized quantization error change is small when the bitwidth is greater than 8 as indicated in Corollary~\ref{coll:bit_error} and Fig.~\ref{fig:quantization_error},
we only consider those bitwidths that are not greater than 8-bit (\ie, $b_j \in \{2,4,8\}$) in our paper.} \jingr{An empirical study on the non-power-of-two bitwidths can be found in Section~\ref{sec:comparison_power_of_two_non_power_of_two}.} 

\subsection{Single-path bit sharing model compression}
\label{sec:bit_sharing_compression}
\subsubsection{Binary gate for quantization}
In a neural network, different layers have diverse redundancy and contribute differently to the accuracy and efficiency of the network.
To determine the bitwidth for each layer,
we introduce a layer-wise binary quantization gate $g_{b_j} \in \{0, 1\} (j > 1)$ on each of the re-assignment offsets in Eq.~(\ref{eq:bit_decompistion}) to encode the choice of the quantization bitwidth $b_j$ as:
\begin{equation}
\begin{aligned}
\label{eq:gate_bit_decomposition}
g_{b_j} &= H\left( || \bz - \bz_{b_{j-1}} ||_1 -  \alpha_{b_{j}} \right), \\
z_{b_K} = z_{b_1} + g_{b_2}& \big( r_{b_2} + \cdots + g_{b_{K-1}} \big (r_{b_{K-1}} + g_{b_K} r_{b_K}  \big) \big) ,
\end{aligned}
\end{equation}
\jingr{
where
$\alpha_{b_{j}} (j > 1)$ is a layer-wise threshold that controls the choice of bitwidth and
$H(A)$ is the step function which returns 1 if $A \ge 0$ and returns 0 otherwise.
Here, 
we use the quantization error 
to determine the choice of bitwidth.
Specifically, if the quantization error is greater than $\alpha_{b_{j}}$, we activate the corresponding quantization gate to increase the bitwidth to reduce the residual error, and vice versa.}

\subsubsection{Binary gate for pruning}
Note that in Eq.~(\ref{eq:gate_bit_decomposition}), we can also consider filter pruning as a special case of quantization, \jing{which makes it possible to perform network pruning and quantization jointly}. To avoid the prohibitively large filter-wise search space, we propose to divide the filters into groups based on channel indices and consider the group-wise \liu{sparsity} instead. 
To be specific, we introduce a binary gate $g_{c, b_1}$ for each group to encode the choice of pruning as:
\begin{equation}
    \begin{aligned}
        \label{eq:ecompistion_compression}
        g_{c, b_1} &= H(|| \bw_c ||_1 - \alpha_{b_1}), \\
        z_{c,b_K} &= g_{c, b_1} \cdot \big( z_{c, b_1}  + g_{b_2} \big( r_{c, b_2} + \cdots ~\\
        &+ g_{b_{K-1}} ( r_{c, b_{K-1}} +  g_{b_K} r_{c, b_K} ) \big) \big),
    \end{aligned}
\end{equation}
\jingr{where 
$z_{c,b_j}$ is an element of the $c$-th group filters with $b_j$-bit quantization and 
$r_{c,b_j}$ is the corresponding re-assignment offset obtained by quantizing the residual error $z_c - z_{c,b_{j-1}}$. 
Here, $\alpha_{b_1}$ is a layer-wise threshold for filter pruning. 
Following PFEC~\cite{li2017pruning}, we use the $\ell_1$-norm criteria to evaluate the importance of \jingr{different groups of filters}. Specifically, if a group of filters is important, the corresponding pruning gate will be activated, and vice versa.
}

\subsubsection{Normalization for binary gate}
\jingr{Note that the binary gates for quantization are layer-wise while those for pruning are group-wise, which may lead to different scales of thresholds. To mitigate the effect of different scaling, we perform normalization before applying the step function $H(\cdot)$. Specifically, given an evaluation metric $A$ ($|| \bz - \bz_{b_{j-1}} ||_1$ for quantization and the $\ell_1$-norm of the $c$-th group of filters
 $||\bw_c||_1 $ for pruning), we obtain the normalized metric $\hat{A}$ by $\hat{A} = A / N_A$, where $N_A$ is the number of elements in $A$. We then feed $\hat{A} - \alpha$ into the step function to obtain the output of the binary gate, where $\alpha$ is the corresponding pruning or quantization threshold.}

\subsection{Learning for loss-aware compression}
\label{sec:learning_compression}
\subsubsection{Gradient approximation for the step function}
\jingr{Instead of manually determining the thresholds of pruning and quantization,
we propose to learn them via \jingr{a gradient descent method}.} \jingr{Unfortunately}, the step function in Eqs.~(\ref{eq:gate_bit_decomposition}) and~(\ref{eq:ecompistion_compression}) is non-differentiable. To address this, \liu{we propose to use straight-through estimator (STE)~\cite{bengio2013estimating,zhou2016dorefa} to approximate the gradient of $H(\cdot)$ by the gradient of the sigmoid function $S(\cdot)$}, which can be formulated as
\begin{equation}
    \label{eq:gradient_approximation}
    \begin{aligned}
    \frac{\partial g}{\partial \alpha} = \frac{\partial H \left( \hat{A} - \alpha \right)}{\partial \alpha} &\approx  \frac{\partial S \left( \hat{A} - \alpha \right)}{\partial \alpha} ~\\
    &= -S \left( \hat{A} - \alpha \right) \left(1 - S(\hat{A} - \alpha)\right),
    \end{aligned}
\end{equation}
where $g$ is the output of a binary gate.

\subsubsection{Objective function for \methodshortname}
\label{sec:objective_lbs}
\jingr{Let $\bW$ be the model parameters and $\boldsymbol{\alpha}$ be the compression configuration that are composed of pruning and quantization thresholds. To design a hardware-efficient model, the objective function should reflect both the accuracy and computation cost of a compressed model.} Following~\cite{cai2018proxylessnas}, we incorporate the computation cost into the objective function and formulate the joint objective as:
\begin{equation}
    \label{eq:constraint_objective_function}
    \mathcal{L}(\bW , \boldsymbol{\alpha}) = \mathcal{L}_{ce}(\bW , \boldsymbol{\alpha}) + \lambda \log R(\boldsymbol{\alpha}),
\end{equation}
\jing{where $\mathcal{L}_{ce}(\cdot,\cdot)$ is the cross-entropy loss, $R(\cdot)$ is the computation cost of the network and}
\revise{$\lambda$ is a hyperparameter that adjusts the importance of the computation cost term $\log R(\boldsymbol{\alpha})$ in the loss function. In fact, a larger $\lambda$ indicates that we put more penalty on computation cost term and thus results in a compressed model with lower resource consumption.}
Following single-path NAS~\cite{stamoulis2019single}, we use a similar formulation of computation cost to preserve the differentiability of the objective function.

For simplicity, we consider performing weight quantization only. 
\jingr{Let $G^l$ be the number of filters groups for layer $l$ and $R_{c,b_j}^l$ be the computation cost of the $c$-th group of filters with $b_j$-bit quantization for layer $l$.}
The computation cost $R(\cdot)$ is formulated as follows:
\begin{equation}
\begin{aligned}
    \label{eq:regularization}
    R(\boldsymbol{\alpha}) &= \sum_{l=1}^L \sum_{c=1}^{G^l} g_{c,b_1}^l \big(R_{c, b_1}^l + g_{b_2}^l \big( R_{c, b_2}^l - R_{c, b_1}^l ~\\ &+ \cdots + g_{b_K}^l (R_{c, b_{K}}^l - R_{c, b_{K-1}}^l ) \big) \big),
\end{aligned}
\end{equation}
\jingr{where $g_{b_j}^l$ and $g_{c, b1}^l$ are the binary gates for $b_j$-bit quantization and the pruning decision for the $c$-th group of filters, respectively.}
Similarly, the computation cost for activation quantization can be easily derived by replacing the binary gates of weights with those of activations.

Note that in Eq.~(\ref{eq:gradient_approximation}), we approximate the gradient of the step function $H(\cdot)$ by the gradient of the sigmoid function $S(\cdot)$. \jingr{Therefore, the objective function in Eq.~(\ref{eq:constraint_objective_function}) remains differentiable.}
\jingr{By minimizing the objective using gradient descent,
the configurations of each layer are automatically determined. Moreover, we are able to make a trade-off between pruning and quantization.} However, the gradient approximation of the binary gate may inevitably introduce noisy signals, which is more severe when we quantize both weights and activations. 

\jingr{To alleviate the noisy signals in the gradient
approximation of the binary gate}, we propose to train the binary gates of weights and activations in an alternating manner. \jingr{Let $\boldsymbol{\alpha}_w$ and $\boldsymbol{\alpha}_x$ be the thresholds for weights and activations quantization, respectively.}
The training algorithm of the proposed \methodshortname is shown in Algorithm~\ref{alg:lbs}. 
\jingr{Starting from a pre-trained model $M^0$, \jingr{we first initialize} $\boldsymbol{\alpha}_w$ and $\boldsymbol{\alpha}_x$ to $\boldsymbol{0}$. Then, we train the compressed model $M$ for $T$ epochs. For each epoch, we train the model parameters $\bW$ with alternating updates of $\boldsymbol{\alpha}_w$ and $\boldsymbol{\alpha}_x$.} 
\jing{An empirical study over the effect of alternating training scheme is put in Section~\ref{sec:effect_algernative_training}.}
\jingr{Once the training is finished, 
we are able to obtain a compressed model by removing unactivated re-assignment offsets and filters.
\jingrs{Following the common practice in model compression~\cite{wang2019haq,wu2018mixed,dong2019hawq,dong2019network,liu2019metapruning}, we then fine-tune the resulting compressed model to compensate for the accuracy loss of model compression.} During fine-tuning, we use the quantization method mentioned in Section~\ref{sec:related_work} with the searched bitwidths rather than the quantization decomposition. Therefore, the inference cost of our \methodshortname is the same as the traditional quantization methods.
}

\jing{
\begin{algorithm}[!t]
\caption{Training method for \methodshortname}
\label{alg:lbs}
\textbf{Input}: A pre-trained model $M^0$, the number of candidate bitwidths $K$, a sequence of candidate bitwidths $\{b_j\}_{j=1}^{K} $, the number of training epochs $T$, 
\jingr{the number of training iteration $I$,} 
and hyperparameters $\lambda$.
\\
\textbf{Output}: A compressed model $M$.
 \begin{algorithmic}[1]
	\STATE Initialize $M$ using $M^0$.
	\STATE \jingr{Initialize the weights quantization thresholds  $\boldsymbol{\alpha}_w$ and activations quantization thresholds $\boldsymbol{\alpha}_x$ to $\boldsymbol{0}$.}
	\FOR{ epoch $t\in \{1,\dots,T\}$}
	    \FOR{iteration $i\in \{1, \dots, I\}$}
    	\STATE \jingr{Calculate the binary gates for quantization and pruning using Eqs.~(\ref{eq:gate_bit_decomposition}) and (\ref{eq:ecompistion_compression}).}
    	\IF{$i = 2n, ~~n \in \mathbb{N}^{+}$}
    	    \STATE Update ${\rm \bW}$ and $\boldsymbol{\alpha}_w$ by minimizing        Eq.~(\ref{eq:constraint_objective_function}).
    	\ELSE
    	    \STATE Update ${\rm \bW}$ and $\boldsymbol{\alpha}_x$ by minimizing
        Eq.~(\ref{eq:constraint_objective_function}).
    	\ENDIF
        \ENDFOR
    \ENDFOR
	\end{algorithmic}
\end{algorithm}
}

\begin{figure}[!t]
	\centering
	\includegraphics[width = 0.61\linewidth]{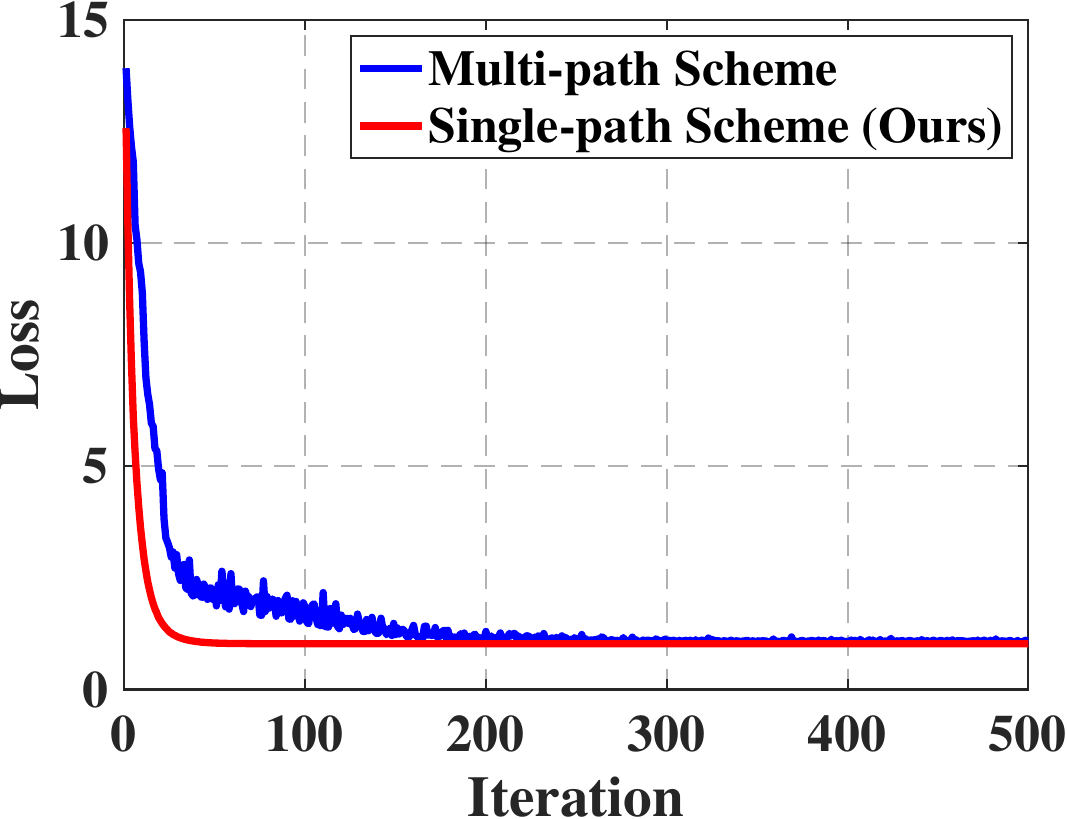}
	\caption{Loss comparison between the single-path (Ours) and multi-path~\cite{wu2018mixed} schemes on a toy dataset with details discussed in Section~\ref{sec:more_discussions}.
	}
	\label{fig:comparisons_loss}
\end{figure}

\vspace{-0.3in}
\subsection{More discussions on \methodshortname}
\label{sec:more_discussions}
Our single-path scheme is different from the multi-path scheme in DNAS~\cite{wu2018mixed}. To study the difference, for simplicity, we consider a linear regression scenario and compare the quantization errors incurred by the two schemes over the regression weights. Formally, we have the following results on the quantization errors incurred by the two schemes.
\begin{prop}\label{prop:error}
Consider a linear regression problem $\min \limits_{\bw \in \sR^{d}} \underset{(\bx, y) \sim \gD}  \sE[(y - \bw\bx)^2]$ with data pairs $\{(\bx, y)\}$, where $\bx \in \sR ^d$ is sampled from $\gN(0,\sigma^2\bf I)$ and $y \in \sR $ is its response. Consider using \methodshortname and DNAS to quantize the linear regression weights.
Let $\bw_L^t$ and $\bw_D^t$ be the quantized regression weights of \methodshortname and DNAS at the iteration $t$ of the optimization, respectively.
Then the following equivalence holds during the optimization process:
\begin{equation}
    \begin{aligned}
    &\lim_{t \rightarrow \infty} \sE _{(\bx, y)\sim\gD}[(y \small{-} \bw^t_{\text{L}}\, \bx)^2] \small{=} \lim_{t \rightarrow \infty} \sE _{(\bx, y)\sim\gD}[(y \small{-} \bw^t_{\text{D}}\, \bx)^2],\\
    &\mathrm{where}~\bw^t_{\text{L}} \small{=}  \bw^t_{L,b_1}\! \!\small{+} g_{b_2}^t \big( \br^t_{b_2} \!\small{+}\! \cdots + \!g_{b_{K-1}}^t \!\big (\br^t_{b_{K-1}} \!\small{+} g_{b_K}^t\! \br^t_{b_K} \! \big)\! \big),\\
    &~~~~~~~~~~~\br^t_{b_{j}} = D(\bw^t_{\text{L}} - \bw^t_{{\text{L}},{b_{j-1}}},s_{b_{j}}),~ j=2,\cdots,K,\\
    &~~~~~~~~~~~\bw^t_{\text{D}} = \sum_{i=1}^K p_i^t\bw_{b_i}^t,~~~\sum_{i=1}^K p^t_i = 1.
     \end{aligned}
\end{equation}
\end{prop}

From Proposition \ref{prop:error}, the multi-path scheme (DNAS) converges to our single-path scheme during the optimization. However, our single-path scheme contains fewer parameters and consumes less computational overhead. 
\jing{
For example, the multi-path scheme maintains $K$ paths while our proposed single path scheme maintains only single-path. Thus, the parameters and computation cost reduction is $(K - 1)/K$.}

\jing{To further study the difference between our single-path scheme and multi-path scheme, we consider a linear regression scenario $\min \limits_{\bw \in \sR^{d}} \underset{(\bx, y) \sim \gD}  \sE[(y - \bw\bx)^2]$ and compare the quantization errors incurred by the two quantization schemes over the regression weights. Specifically, we construct a toy linear regression dataset $\{(\bx_i, y_i)\}_{i=1}^N$ where we randomly sample $N\small{=}10,000$ data $\bx \in \sR^{10}$ from Gaussian distribution $\gN(0, \bf I)$ and obtain corresponding responses $y = \bw^*\bx + \Delta$. Here, $\bw^*\in \sR^{10}$ is a fixed weight randomly sampled from $[0,1]^{10}$ and $\Delta\in \sR$ is a noise sampled from $\gN(0, 1)$. We then apply \methodshortname and DNAS to quantize the regression weights. As shown in Fig.~\ref{fig:comparisons_loss}, our method converges faster and smoother than the multi-path scheme.}

\begin{table*}[t]
\renewcommand{\arraystretch}{1.3}
\caption{Comparisons of different methods w.r.t. Bit-Operation (BOP) count on CIFAR-100. ``BOP comp. ratio'' denotes the BOP compression ratio.
}
\vspace{-0.1in}
\centering
\scalebox{0.88}
{
\begin{tabular}{cccccccc}
\toprule
Network  & Method & BOPs (M) & \tabincell{c}{BOP comp. ratio} & \tabincell{c}{Search Cost \\ (GPU hours)} & Top-1 Acc. (\%) & Top-5 Acc. (\%) \\
\midrule
\multirow{8}{*}{ResNet-20} & Full-precision & 41798.6 & 1.0 & -- & 67.5 & 90.8 \\
& 4-bit precision & 674.6 & 62.0 & -- & 67.8$\pm$0.3 & 90.4$\pm$0.2 \\
\cdashline{2-7}
& DQ~\cite{Uhlich2020Mixed} & 1180.0 & 35.4 & 2.2 & 67.7$\pm$0.6 & 90.4$\pm$0.5 \\
& HAQ~\cite{wang2019haq} & 653.4 & 64.0 & 5.8 & 67.7$\pm$0.1 & 90.4$\pm$0.3 \\
& DNAS~\cite{wu2018mixed} & 660.0 & 62.9 & 2.8& 67.8$\pm$0.3 & 90.4$\pm$0.2 \\
\cdashline{2-7}
& \methodshortname-P (Ours) & 28586.5 & 1.5 & \textbf{0.2} & 67.9$\pm$0.1 & 90.7$\pm$0.2 \\
& \methodshortname-Q (Ours) & 649.5 & 64.4 & 0.8 & 68.1$\pm$0.1 & 90.5$\pm$0.0 \\
& \methodshortname (Ours) & \textbf{630.6} & \textbf{66.3} & 1.0 & \textbf{68.1$\pm$0.3} & \textbf{90.6$\pm$0.2}\\
\midrule
\multirow{8}{*}{ResNet-56} & Full-precision & 128771.7 & 1.0 & -- &  71.7 & 92.2 \\
& 4-bit precision & 2033.6 & 63.3 & -- & 70.9$\pm$0.3 & 91.2$\pm$0.4 \\
\cdashline{2-7}
& DQ~\cite{Uhlich2020Mixed} & 2222.9 & 57.9 & 7.0 & 70.7$\pm$0.2 & 91.4$\pm$0.4\\
& HAQ~\cite{wang2019haq} & 2014.9 & 63.9 & 12.9 & 71.2$\pm$0.1 & 91.1$\pm$0.2\\
& DNAS~\cite{wu2018mixed} & 2016.8 & 63.8 & 7.6 & 71.2$\pm$0.1 & 91.3$\pm$0.2 \\
\cdashline{2-7}
& \methodshortname-P (Ours) & 87021.6 & 1.5 & \textbf{0.4} & 71.5$\pm$0.1 & 91.8$\pm$0.2 \\
& \methodshortname-Q (Ours) & 1970.7 & 65.3 & 1.3 & 71.5$\pm$0.2 & 91.5$\pm$0.2\\
& \methodshortname (Ours) & \textbf{1918.8} & \textbf{67.1} & 1.5 & \textbf{71.6$\pm$0.1} & \textbf{91.8$\pm$0.4} \\
\bottomrule
\end{tabular}
}
\label{table:results_on_CIFAR_100}
\vspace{-0.1in}
\end{table*}

\section{Experiments}
\label{sec:experiments}

\begin{table*}[!tb]
\renewcommand{\arraystretch}{1.3}
\caption{
\zheng{Comparisons of different methods w.r.t. memory footprints on CIFAR-100. ``M.f. comp. ratio" denotes the memory footprints compression ratio.
}
}
\vspace{-0.1in}
\centering
\scalebox{0.88}
{
\begin{tabular}{ccccccc}
\toprule
Network & Method & Memory footprints (KB) & M.f. comp. ratio & Top-1 Acc. (\%) & Top-5 Acc. (\%)  \\ 
\midrule
\multirow{8}{*}{ResNet-56} & Full-precision & 5653.4 & 1.0 & 71.7  & 92.2   \\
&4-bit precision & 711.7 & 7.9 & 70.9$\pm$0.3 & 91.2$\pm$0.4   \\
\cdashline{2-7}
& DQ~\cite{Uhlich2020Mixed} & 723.2 & 7.8 & 70.9$\pm$0.4 & 91.7$\pm$0.3 \\
& HAQ~\cite{wang2019haq} & 700.0 & 8.1 & 71.3$\pm$0.1  & 91.1$\pm$0.1   \\
& DNAS~\cite{wu2018mixed}  & 708.9  & 8.0 & 71.5$\pm$0.2  &  91.3$\pm$0.1  \\
\cdashline{2-7}
&\methodshortname-P (Ours) & 4778.0 & 1.2 & 71.5$\pm$0.1 & 91.8$\pm$0.2 \\
&\methodshortname-Q (Ours) &  674.5 & 8.4 & 71.5$\pm$0.2  &  91.6$\pm$0.2  \\
&\methodshortname (Ours) & \textbf{657.3} & \textbf{8.6} & \textbf{71.6$\pm$0.1}  &  \textbf{91.8$\pm$0.4}  \\
\bottomrule
\end{tabular}
}
\label{table:results_on_cifar100_Memory}
\vspace{-0.1in}
\end{table*}

\subsection{Datasets and evaluation metrics}
\label{sec:dataset_evaluation_metrics}
\jing{We evaluate the proposed \methodshortname on two image classification datasets, including CIFAR-100~\cite{krizhevsky2009learning} and ImageNet~\cite{deng2009imagenet}. CIFAR-100 consists of 50k training samples and 10k testing images with 100 classes. ImageNet contains 1.28 million training samples and 50k testing images for 1,000 classes.}

We measure the performance of different methods using the Top-1 and Top-5 accuracy. Experiments on CIFAR-100 are repeated 5 times and we report the mean and standard deviation. 
\jing{For fair comparisons, we measure the computation cost by the Bit-Operation (BOP) count for all the compared methods following~\cite{guo2020single,wang2020diff}.} The BOP compression ratio is defined as the ratio between the total BOPs of the uncompressed and compressed model. \liu{We can also measure the computation cost by the total weights and activations memory footprints \zheng{following DQ~\cite{Uhlich2020Mixed}}.} Similarly, the memory footprints compression ratio is defined as the ratio between the total memory footprints of the uncompressed and compressed model.
\zheng{Moreover, following~\cite{stamoulis2019single,liu2018darts}, we use the search cost on a GPU device (NVIDIA TITAN Xp) to measure the time of finding an optimal compressed model.}

\subsection{Implementation details}
\label{sec:implementation_details}
To demonstrate the effectiveness of the proposed method, we apply \methodshortname to various architectures, such as ResNet~\cite{he2016deep} and MobileNetV2~\cite{sandler2018inverted}. 
All implementations are based on PyTorch~\cite{paszke2019pytorch}.
Following HAQ~\cite{wang2019haq}, we quantize all the layers, in which the first and the last layers are quantized to 8-bit. Following ThiNet~\cite{luo2019thinet}, \jing{we only perform filter pruning for the first layer in the residual block.} For ResNet-20 and ResNet-56 on CIFAR-100~\cite{krizhevsky2009learning}, we set $B$ to 4. For ResNet-18 and MobileNetV2 on ImageNet~\cite{russakovsky2015imagenet}, $B$ is set to 16 and 8, respectively. \revise{We tune $\lambda$ to obtain compressed models under different resource constraints.}
We first train the full-precision models and then use the pre-trained weights to initialize the compressed models \jingrs{following~\cite{wang2019haq,dong2019hawq}}.
For CIFAR-100, we use the same data augmentation as in~\cite{he2016deep}, including randomly cropping and horizontal flipping. For ImageNet, images are resized to $256 \times 256$, and then a $224 \times 224$ patch is randomly cropped from an image or its horizontal flip for training. For testing, a $224 \times 224$ center cropped is chosen.

Following~\cite{Li2020Additive}, we introduce weight normalization during training. We use SGD with nesterov~\cite{nesterov1983method} for optimization. The momentum term is set to $0.9$.
We first \jingrs{search configurations} for 30 epochs on CIFAR-100 and 10 epochs on ImageNet.  
The learning rate is set to $0.001$. We then fine-tune the searched compressed network to recover the performance drop. On CIFAR-100, we fine-tune the searched network for $200$ epochs with a mini-batch size of $128$. 
The learning rate is initialized to $0.1$ and is divided by 10 at $80$-th and $120$-th epochs.
For ResNet-18 on ImageNet, we fine-tune the searched network for $15$ epochs with a mini-batch size of 256. For MobileNetV2 on ImageNet, we fine-tune for 150 epochs. For all models on ImageNet, the learning rate starts at 0.01 and decays with cosine annealing~\cite{loshchilov2016sgdr}. 

\subsection{Compared methods}
To investigate the effectiveness of \methodshortname, we consider the following methods for comparisons: \textbf{\methodshortname}: our proposed method with joint pruning and quantization; \textbf{\methodshortname-Q}: \methodshortname with quantization only; \textbf{\methodshortname-P}: \methodshortname with pruning only; and several state-of-the-art model compression methods: \jing{\textbf{DNAS}~\cite{wu2018mixed}: \hk{it} uses multi-path search scheme to search for optimal bitwidths; \textbf{HAQ}~\cite{wang2019haq}: \hk{it} uses reinforcement learning to automatically determine the quantization policies; \textbf{HAWQ}~\cite{dong2019hawq}: \hk{it} uses second-order Hessian information
to guide the bitwidth search of each layer;
\textbf{DQ}~\cite{Uhlich2020Mixed}: \hk{it} uses a gradient-based method to learn the bitwidths; \textbf{DJPQ}~\cite{wang2020diff}: \hk{it} combines the variational information bottleneck method to structured pruning and mixed-bit precision quantization and learns optimal configurations using a gradient-based method; \textbf{Bayesian Bits}~\cite{van2020bayesian}: \hk{the authors cast} configuration search problem of pruning and quantization into a variational inference problem and use gradient-based method to learn configurations. We do not consider other methods since they have been compared in the methods mentioned above. \jingrs{Except DQ, all the compared methods use the same strategy that first conducts configurations search based on the same pre-trained model and then finetunes the resulting model with the searched configurations.}}

\subsection{Comparisons on CIFAR-100}
We apply \methodshortname to compress ResNet-20 and ResNet-56 on CIFAR-100.
We report the results under BOPs and memory footprints constraints in Tables~\ref{table:results_on_CIFAR_100} and~\ref{table:results_on_cifar100_Memory}. Compared with the full-precision counterparts, 4-bit quantized models achieve comparable performance or even better performance. \jingrs{This can be attributed to the redundancy removal and regularization effect of network quantization. Similar phenomenon is also observed in LSQ~\cite{Esser2020LEARNED}.} 
Compared with fixed-precision models, mixed-precision methods are able to further reduce the BOPs while preserving performance. 
Critically, \methodshortname-Q outperforms state-of-the-arts DQ, HAQ, and DNAS with less computation cost and memory footprints. For example, 
\methodshortname-Q ResNet-56 outperforms DNAS by $0.3\%$ on the Top-1 accuracy while achieving 65.3$\times$ BOPs reduction. 
\jing{By performing pruning and quantization jointly, \methodshortname achieves the best performance while further reducing computation cost and memory footprints of the compressed models.}

We also show the results of the compressed ResNet-56 with different BOPs and memory footprints in Fig.~\ref{fig:comparisons_different_bops}. \jing{From the results, \methodshortname consistently outperforms all the other methods under variant BOPs and memory footprints settings. Moreover, our proposed \methodshortname achieves significant improvement in terms of BOPs and memory footprints compared with fixed-precision quantization, especially at low BOPs and memory footprints settings.} For example, compared with the fixed-precision counterpart, \jingrs{\methodshortname compressed ResNet-56 consumes much less computational overhead (783.52 \vs~1156.46 BOPs) and fewer memory footprints (395.25 \vs~536.24 KB) but achieves comparable performance.}

\begin{figure}[!t]
	\centering
	\subfigure[BOPs \vs~Top-1 Acc.]{
		\includegraphics[width = 0.43\linewidth]{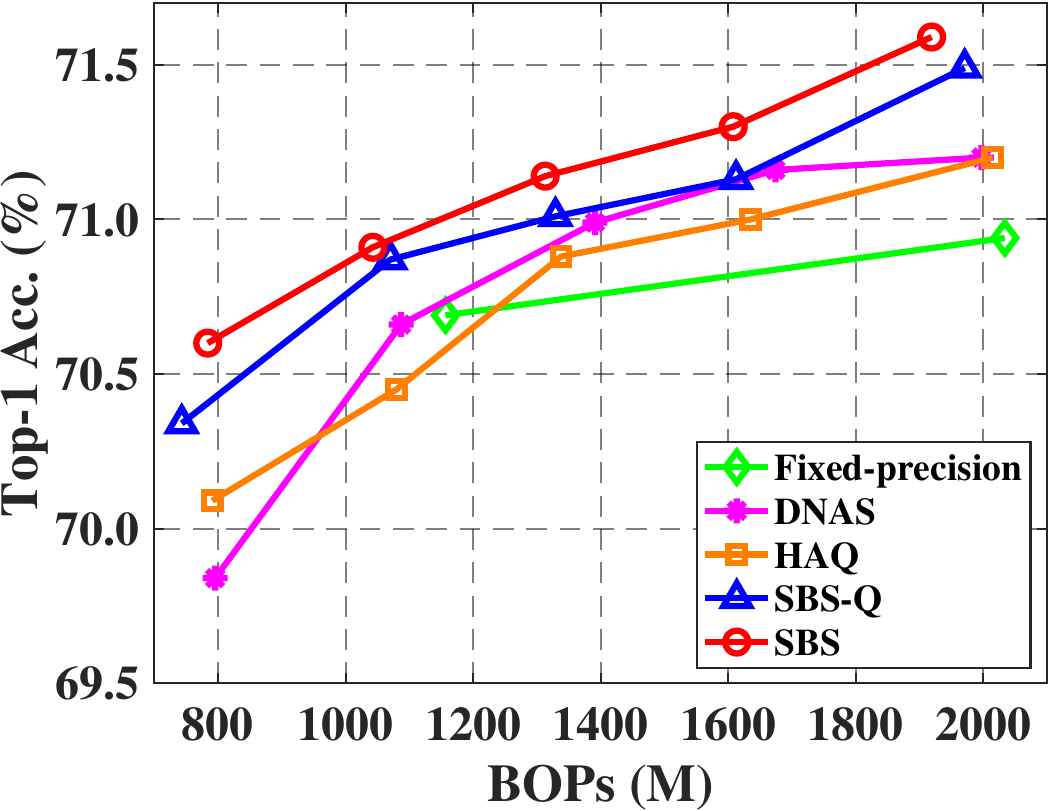}
		\label{fig:resnet56_bops_acc}
	}~~
    \subfigure[Memory \vs~Top-1 Acc.]{
		\includegraphics[width = 0.43\linewidth]{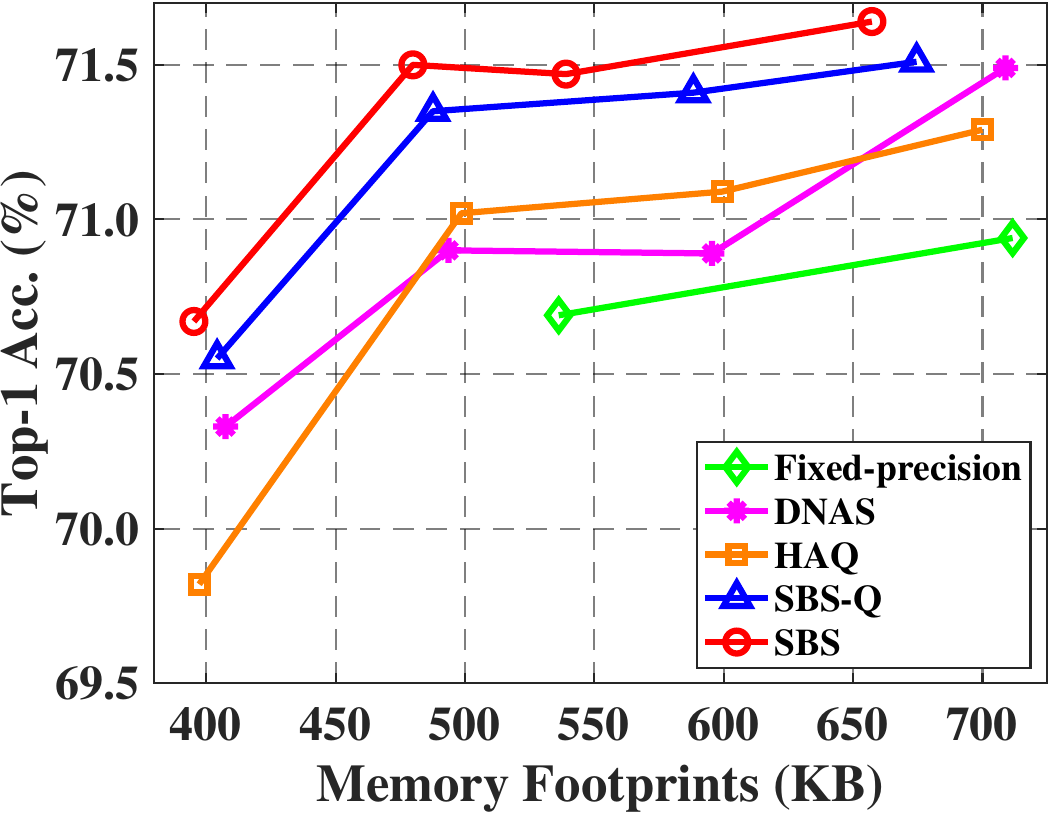}
		\label{fig:resnet56_footprint_acc}
	}~~
	\caption{\jing{Performance comparisons of different methods with different BOPs and memory footprints. We use different methods to compress ResNet-56 and report the results on CIFAR-100.}}
	\label{fig:comparisons_different_bops}
	\vspace{-2pt}
    \vspace{-0.1in}
\end{figure}

\jing{To evaluate the efficiency of the proposed method, we also compare the search cost of different methods. 
From Table~\ref{table:results_on_CIFAR_100},
the search cost of the proposed \methodshortname is much smaller than the state-of-the-art methods. For example, 
for ResNet-20, the search cost of \methodshortname is 2.8$\times$ lower than DNAS while 
the search cost of \methodshortname for ResNet-56 is nearly 5.1$\times$ lower than DNAS. Compared with \methodshortname-Q, \methodshortname only introduces a small amount of overhead (\eg, 0.2 GPU hour for ResNet-56).
These results show the superior efficiency of \methodshortname-Q and \methodshortname.}

\begin{table}[!t]
    \renewcommand{\arraystretch}{1.3}
    \caption{
    Comparisons between our proposed method and fixed-precision quantization on ImageNet.
    }
    \vspace{-0.1in}
    \begin{center}
    \scalebox{0.8}
    {
    \begin{tabular}{cccccccc}
    \toprule
    Network  & Method & Top-1 Acc. (\%) & Top-5 Acc. (\%) & BOPs (G)  \\ 
    \midrule
    \multirow{7}{*}{ResNet-18} & Full-precision & 69.8 & 89.1 & 1857.6  \\
    \cdashline{2-7}
    & 8-bit precision & 70.0 & 89.3  & 116.1 \\
    & \methodshortname-Q (Ours) & 70.2 & 89.4  & 114.3 \\
    & \methodshortname (Ours)  & \textbf{70.3} & \textbf{89.4} & \bf{108.4} \\
    \cdashline{2-7}
    & 4-bit precision & 69.2 & 89.0 & 34.7  \\
    & \methodshortname-Q (Ours) & 69.5 & 89.0 & 34.4  \\
    & \methodshortname (Ours) & \textbf{69.6} & \textbf{89.0} & \textbf{34.0} \\
    \midrule
    \multirow{7}{*}{ResNet-50} & Full-precision & 76.8 & 93.3 & 4187.3  \\
    \cdashline{2-7}
    & 8-bit precision &  76.3 & 93.0 & 261.7  \\
    & \methodshortname-Q (Ours)  & 76.4 & 93.1 & 253.1 \\
    & \methodshortname (Ours)  & \textbf{76.5} & \textbf{93.1} & \textbf{243.9} \\
    \cdashline{2-7}
    & 4-bit precision  & 75.7 & 92.7 & 71.2 \\
    & \methodshortname-Q (Ours)  &  75.8 & 92.8 & 71.0 \\
    & \methodshortname (Ours)  & \textbf{75.9} & \textbf{92.8} & \textbf{70.4} \\
    \bottomrule
    \end{tabular}
    }
    \end{center}
    \label{table:mixed_vs_fixed_precision}
    \vspace{-0.1in}
\end{table}

\begin{table*}[!t]
\renewcommand{\arraystretch}{1.3}
\caption{
Comparisons of different methods on ImageNet. \zheng{“*” denotes that we get the results from the figures in~\cite{van2020bayesian} and} ``--" denotes that the results are not reported. \jing{``BOP comp. ratio'' denotes the BOP compression ratio.}
}
\vspace{-0.1in}
\begin{center}
\scalebox{0.9}
{
\begin{tabular}{cccccccc}
\toprule
Network  & Method & BOPs (G) & BOP comp. ratio & Top-1 Acc. (\%) & Top-5 Acc. (\%) & \\ 
\midrule
\multirow{9}{*}{ResNet-18} & Full-precision & 1857.6 & 1.0 & 69.8 & 89.1 \\
& 4-bit precision & 34.7 & 53.5 & 69.2 & 89.0 \\
\cdashline{2-7}
& DQ~\cite{Uhlich2020Mixed} & 40.1 & 46.3& 68.9 & 88.6\\
& \zheng{Bayesian Bits*}~\cite{van2020bayesian} & 35.9 & 51.7 & 69.5 & -- \\
& DJPQ~\cite{wang2020diff} & 35.5 & 52.3 & 69.1 & --  \\
& HAQ~\cite{wang2019haq} & 34.4 & 54.0 & 69.2 & 89.0 \\
& HAWQ~\cite{dong2019hawq} & 34.0 & 54.6 & 68.5 & -- \\
\cdashline{2-7}
& \methodshortname-Q (Ours) & 34.4 & 54.0 & 69.5 & 89.0 \\
& \methodshortname (Ours) & \textbf{34.0} & \textbf{54.6} & \textbf{69.6} & \textbf{89.0}\\
\midrule
\multirow{7}{*}{MobileNetV2} & Full-precision & 308.0 & 1.0 &  71.9 & 90.3 \\
& 8-bit precision & 19.2 & 16.0 & 71.5 & 90.1 \\
\cdashline{2-7}
& DQ~\cite{Uhlich2020Mixed} & 19.6 & 1.9 & 70.4 & 89.7 \\
& HAQ~\cite{wang2019haq} & 13.8 & 22.3 &  71.4 & 90.2 \\
& \zheng{Bayesian Bits*}~\cite{van2020bayesian} & 13.8 & 22.3 & 71.4 & -- \\
\cdashline{2-7}
& \methodshortname-Q (Ours) & 13.6 & 22.6 & 71.6 & 90.2 \\
& \methodshortname (Ours) & \textbf{13.6} & \textbf{22.6} & \textbf{71.8}& \textbf{90.3}\\
\bottomrule
\end{tabular}
}
\end{center}
\label{table:results_on_imagenet}
\vspace{-0.1in}
\end{table*}

\subsection{Comparisons on ImageNet}
\revise{To evaluate the effectiveness of our method, we first apply SBS and SBS-Q to compress ResNet-18 and ResNet-50 and evaluate the performance on ImageNet. From Table~\ref{table:mixed_vs_fixed_precision}, SBS-Q compressed models with lower BOPs achieve better performance than the fixed-precision counterparts. For example, for ResNet-18, SBS-Q with 1.8G fewer BOPs outperforms 8-bit quantization by 0.2\% on the Top-1 accuracy. These results justify the effectiveness and necessity of bitwidth search.
By combining pruning and quantization, SBS yields compressed models with higher accuracy and lower BOPs. For example, for ResNet-50, SBS outperforms SBS-Q by 0.1\% on the Top-1 accuracy while reducing 9.2G BOPs. We also show the results of compressed ResNet-18 and ResNet-50 with different BOPs in Fig.~\ref{fig:comparisons_different_bops_imagenet}. Compared with SBS-Q, \jingrs{SBS achieves better accuracy-BOPs trade-off},
which shows the benefit of performing pruning and quantization jointly.}

\revise{
To compare \methodshortname with other state-of-the-art methods, 
we apply different methods to compress ResNet-18 and MobileNetV2. 
From 
Table~\ref{table:results_on_imagenet},
\methodshortname-Q with less computation cost outperforms the state-of-the-art baselines. Specifically, \methodshortname-Q compressed MobileNetV2 surpasses the one compressed by HAQ with more BOPs reduction. 
By combining pruning and quantization, \methodshortname further improves the performance while 
reducing the computation cost of the compressed models. 
For example, \methodshortname compressed MobileNetV2 reduces the BOPs by 22.6$\times$ while only resulting in 0.1\% performance degradation in terms of the Top-1 accuracy.}

\begin{table*}[!htb]
\renewcommand{\arraystretch}{1.3}
\caption{Effect of the bit sharing scheme. We report the Top-1/Top-5 accuracy, BOPs, search cost, and GPU memory footprints on CIFAR-100. The search cost and GPU memory are measured on a GPU device (NVIDIA TITAN Xp).}
\vspace{-0.1in}
\begin{center}
\scalebox{0.9}
{
\begin{tabular}{ccccccc}
\toprule
Network  & Method & Top-1 Acc. (\%) & Top-5 Acc. (\%) & BOPs (M) & Search Cost (GPU hours) & GPU Memory (GB) \\ 
\midrule
\multirow{2}{*}{ResNet-20} & w/o bit sharing & 67.8$\pm$0.1 & 90.5$\pm$0.2 & 664.2 & 2.8 & 4.4 \\
&  w/ bit sharing & \textbf{68.1$\pm$0.1} & \textbf{90.5$\pm$0.0} & \textbf{649.5} & \textbf{0.8} & \textbf{1.5} \\
\midrule
\multirow{2}{*}{ResNet-56} & w/o bit sharing & 71.3$\pm$0.3 & 91.4$\pm$0.4 & 2001.1 & 8.7 & 10.9\\
& w/ bit sharing &  \textbf{71.5$\pm$0.2} &  \textbf{91.5$\pm$0.2} &  \textbf{1970.7} &  \textbf{1.3} & \textbf{3.0}\\
\bottomrule
\end{tabular}
}
\end{center}
\label{table:effect_bit_sharing}
\vspace{-0.2in}
\end{table*}

\begin{figure}[!t]
	\centering
	\subfigure[BOPs \vs~Top-1 Acc. of R18.]{
		\includegraphics[width = 0.43\linewidth]{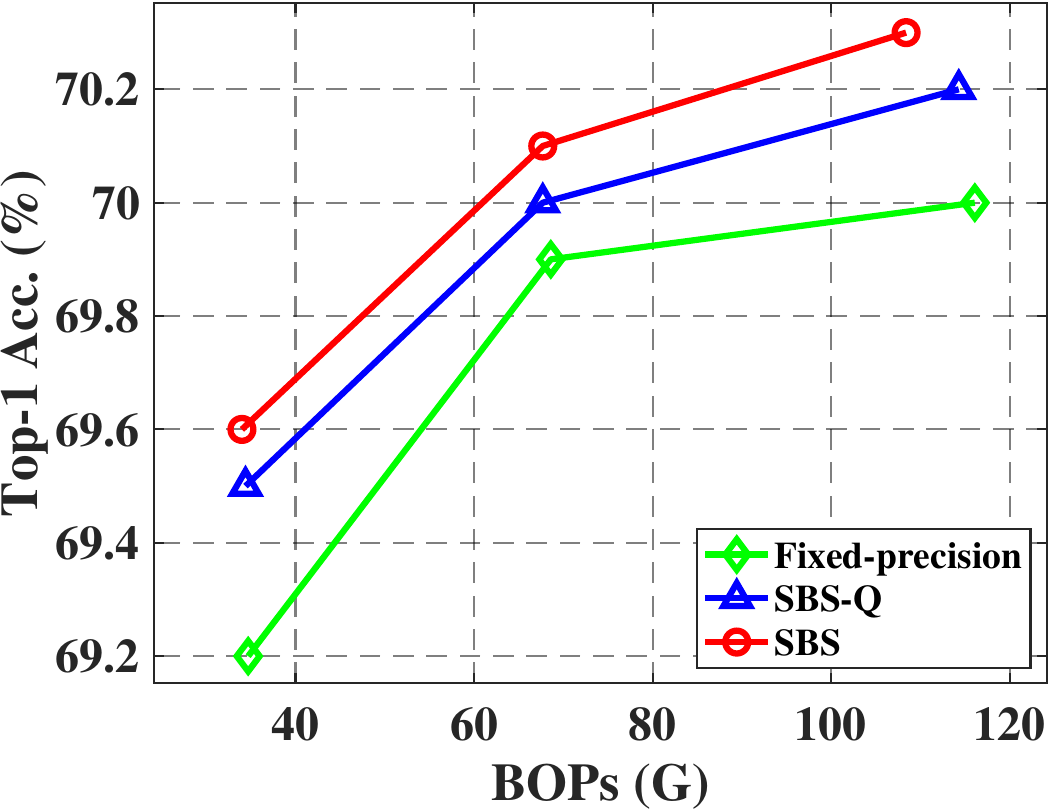}
		\label{fig:resnet18_bops_acc}
	}~~
    \subfigure[BOPs \vs~Top-1 Acc. of R50.]{
		\includegraphics[width = 0.43\linewidth]{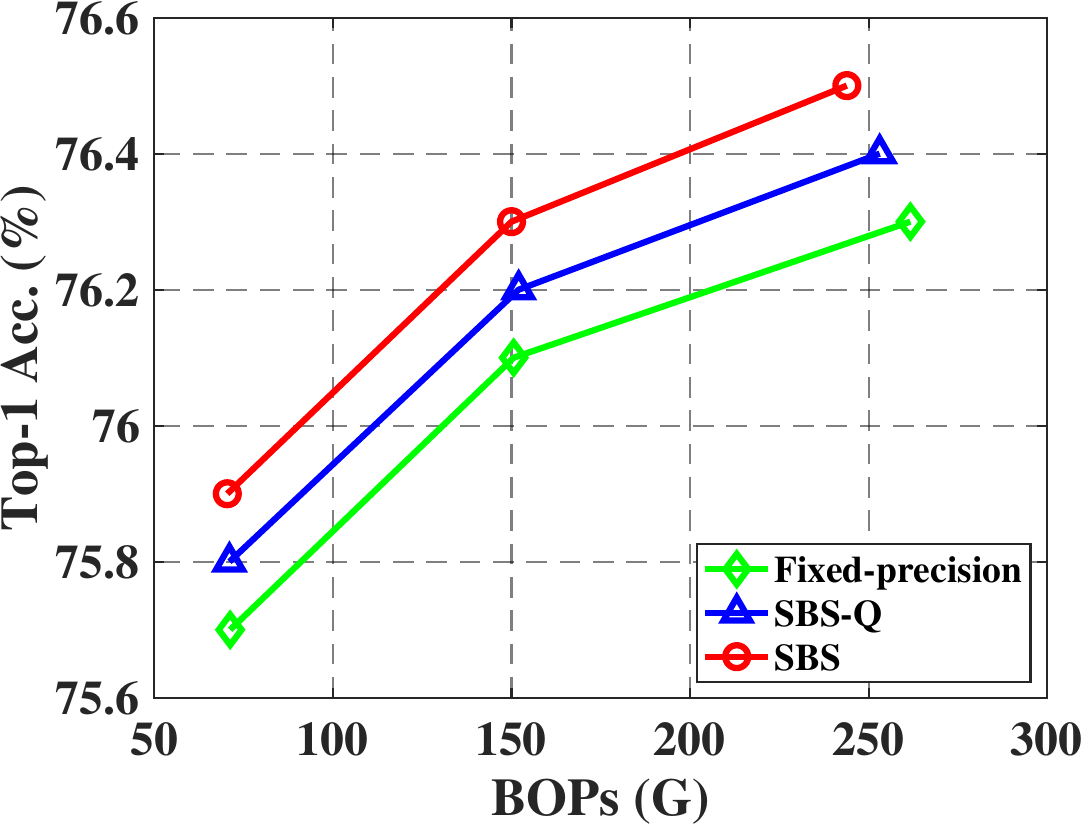}
		\label{fig:resnet50_bops_acc}
	}~~
	\caption{\jing{Performance comparisons of various methods for ResNet-18 and ResNet-50 under different BOPs on ImageNet.}}
	\label{fig:comparisons_different_bops_imagenet}
	\vspace{-10pt}
\end{figure}

\revise{We also illustrate the detailed configurations (\ie, bitwidth and pruning rate) of each layer from the compressed ResNet-18 and ResNet-50. From  Figs.~\ref{fig:config_r18_main} and~\ref{fig:config_r50_main}, \methodshortname assigns higher bitwidth to $1 \times 1$ convolutional layers (including downsampling layers).
One possible reason is that the number of parameters and computation cost of $1 \times 1$ convolutional layers are much smaller than other layers. Compressing these layers may lead to a significant performance decline. Besides, \methodshortname allocates fewer bitwidth to $3 \times 3$ convolutional layers to reduce BOPs. \jingrs{Note that the bit-width allocation might vary among weights, activations, and various models due to distinct tensor dimensions or network building blocks (non-bottleneck for ResNet-18 and bottleneck for ResNet-50), which affects layer complexity and compression trade-offs.}
For filter pruning, 
\methodshortname inclines to prune more filters in the shallower layers of ResNet-18 and middle layers of ResNet-50, which significantly reduces the number of parameters and computational overhead. More detailed configurations of the other models are put in the supplementary material.}

\subsection{Hardware resource-constrained compression}
\label{sec:comparisons_resource_constrained_compression}
To investigate the effect of \methodshortname on hardware devices, \jingrs{we further apply \methodshortname to compress MobileNetV2 under various resource constraints on the BitFusion architecture~\cite{sharma2018bit}, which is a state-of-the-art spatial ASIC accelerator for neural networks.}
We measure the computation cost by the latency and energy on a simulator of the BitFusion with a batch size of 16. We report the results on ImageNet in Table~\ref{table:resource_constrained_compression}. 
Compared with fixed-precision quantization, \methodshortname achieves better performance with lower latency and energy. For example, 
\methodshortname compressed MobileNetV2 with 3.3ms lower latency outperforms 8-bit MobileNetV2 by 0.5\% on the Top-1 accuracy. \jingrs{These results show the promising hardware efficiency of our SBS.}

\begin{table*}[!t]
\renewcommand{\arraystretch}{1.3}
\caption{\jing{Hardware resource-constrained compression on BitFusion. We evaluate the proposed \methodshortname under the latency and energy constraints and report the Top-1 and Top-5 accuracy on ImageNet.}
}
\vspace{-0.1in}
\begin{center}
\scalebox{0.9}
{
\begin{tabular}{cccccccc}
\toprule
\multirow{2}[0]{*}{Network}  & \multirow{2}[0]{*}{Method} & \multicolumn{3}{c}{Latency-constrained} & \multicolumn{3}{c}{Energy-constrained} \\
\cline{3-8}
& & Top-1 Acc. (\%) & Top-5 Acc. (\%) & Latency (ms) & Top-1 Acc. (\%) & Top-5 Acc. (\%) & Energy (mJ) \\
\midrule
\multirow{2}{*}{MobileNetV2} & 8-bit precision & 71.5 & 90.1 & 24.9 & 71.5 & 90.1 & 37.0\\
& \methodshortname (Ours) & \textbf{72.0} & \textbf{90.4} & \textbf{21.6} & \textbf{71.7} & \textbf{90.3} & \textbf{29.8} \\
\bottomrule
\end{tabular}
}
\end{center}
\label{table:resource_constrained_compression}
\vspace{-0.1in}
\end{table*}

\section{Further Studies}
\revise{In this section, we conduct further studies for our \methodshortname. 
1) We investigate the effect of the bit sharing scheme in Section~\ref{sec:effect_bit_sharing}. 
2) We explore the effect of the one-stage compression scheme in Section~\ref{sec:effect_one_stage_compression}. 
3) We study the effect of the alternating training scheme in Section~\ref{sec:effect_algernative_training}. 
4) We explore the effect of different group sizes in Section~\ref{sec:effect_group_size}.
5) We investigate the effect of SBS with non-power-of-two bitwidths in Section~\ref{sec:comparison_power_of_two_non_power_of_two}.
6) We compare the training from scratch scheme with the fine-tuning strategy in Section~\ref{sec:fine_tuning}.}

\begin{figure}[!tb]
    \centering
    \subfigure[\methodshortname searched bitwidths of ResNet-18.]{
        \includegraphics[width=0.89\linewidth]{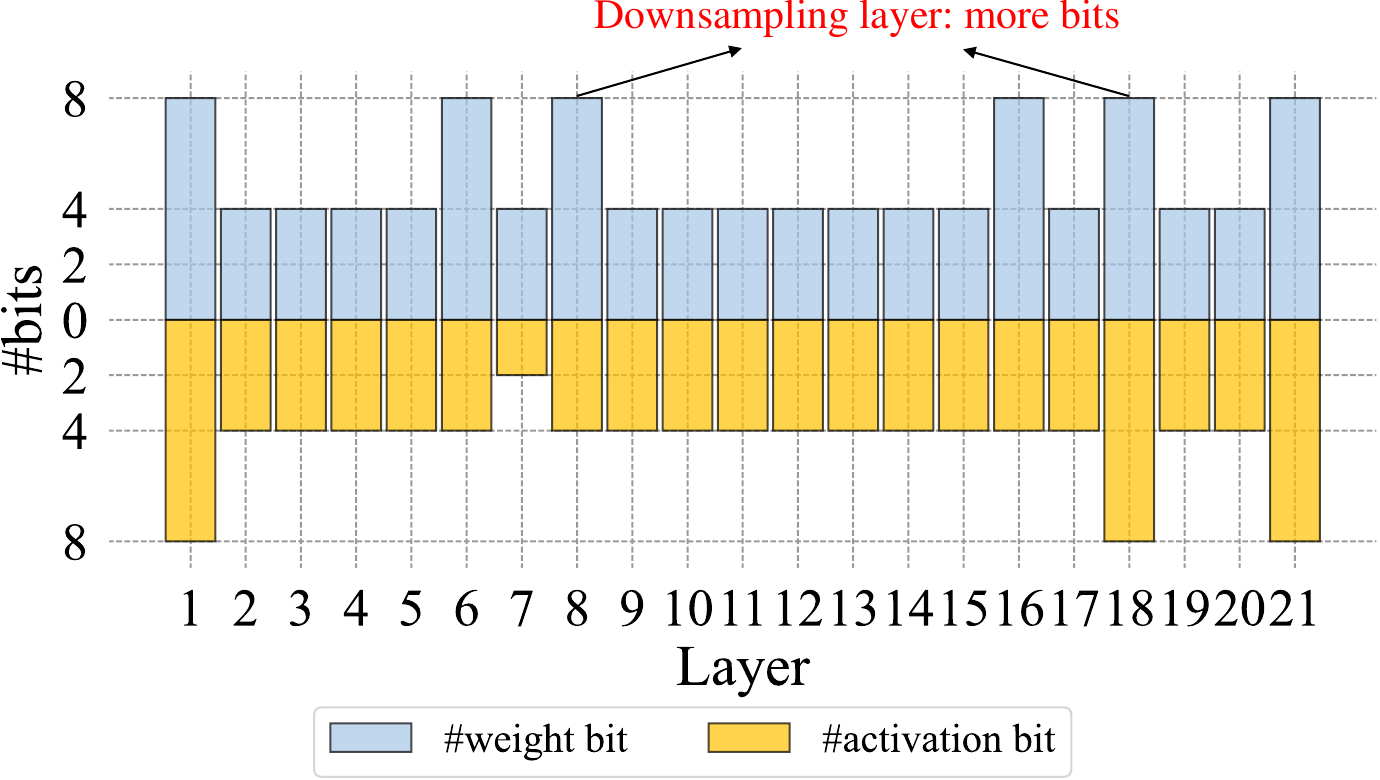}
        \label{fig:bitwidth_r18_main}
    }
    \hspace{3mm}
    \subfigure[\methodshortname searched pruning rates of ResNet-18.]{
        \includegraphics[width=0.89\linewidth]{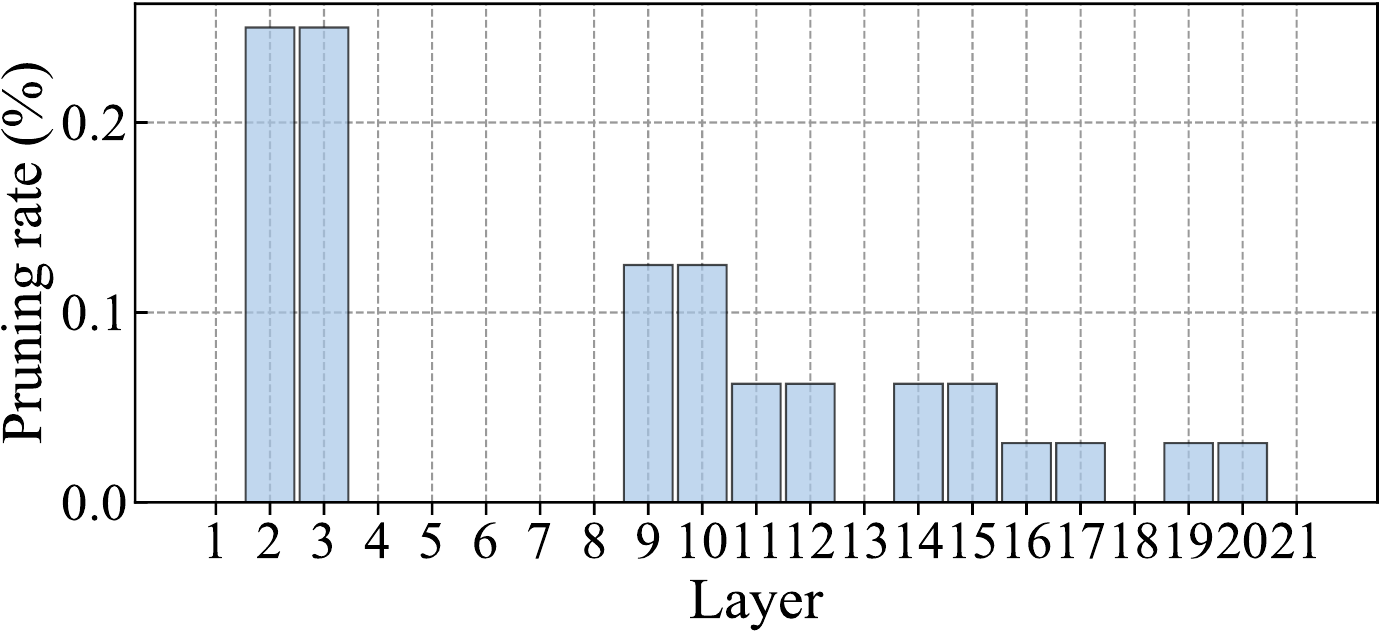}
        \label{fig:pruning_rate_r18_main}
    }
    \caption{\revise{\methodshortname searched configurations of ResNet-18 on ImageNet. The Top-1 accuracy, Top-5 accuracy and BOPs of the compressed ResNet-18 are 69.6\%, 89.0\% and 34.0G, respectively.}
    }
    \label{fig:config_r18_main}
    \vspace{-0.1in}
\end{figure}

\begin{figure}[!tb]
    \centering
    \subfigure[\methodshortname searched bitwidths of ResNet-50.]{
        \includegraphics[width=0.95\linewidth]{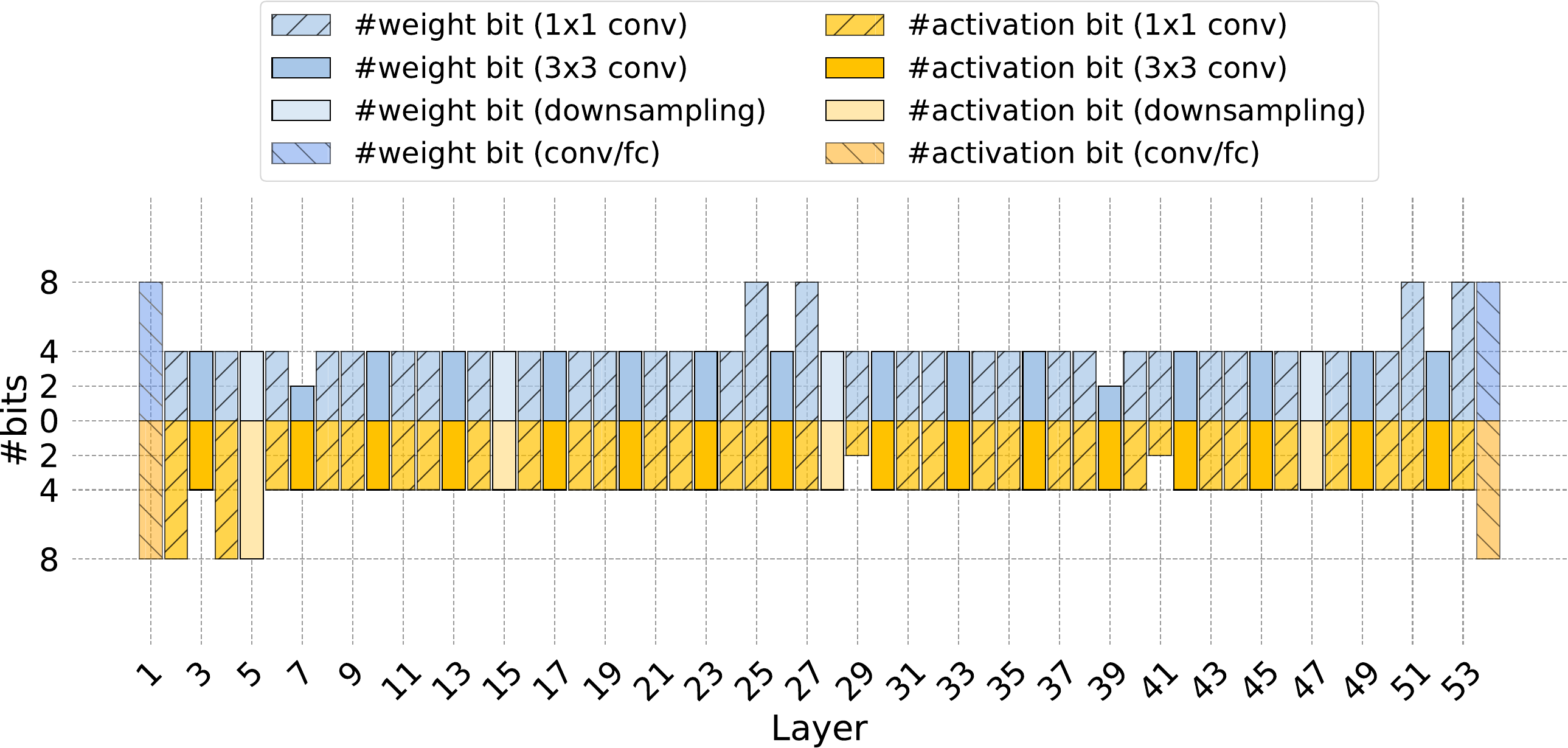}
        \label{fig:bitwidth_r50_main}
    }
    \hspace{3mm}
    \subfigure[\methodshortname searched pruning rates of ResNet-50.]{
        \includegraphics[width=0.95\linewidth]{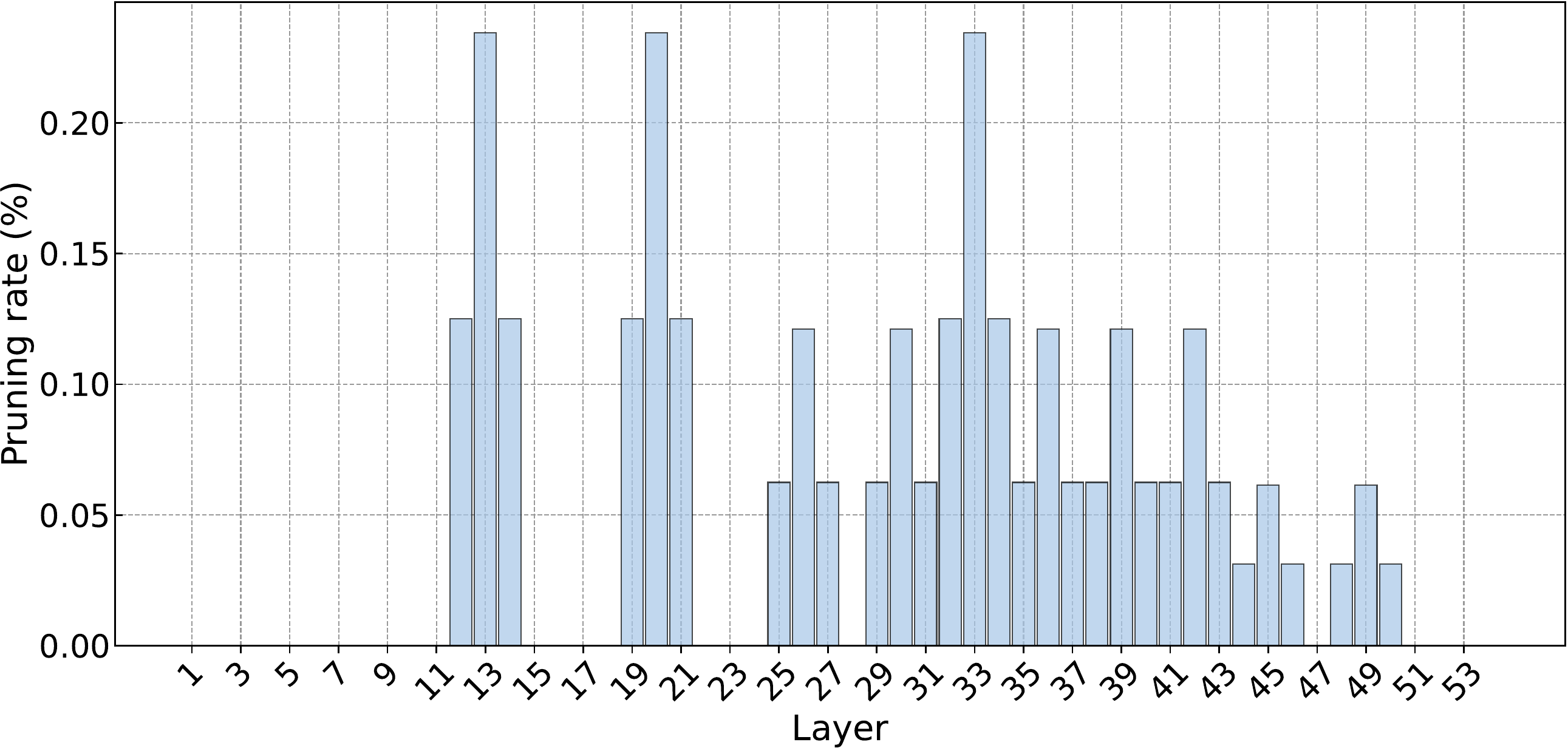}
        \label{fig:pruning_rate_r50_main}
    }
    \caption{\revise{\methodshortname searched configurations of ResNet-50 on ImageNet. The Top-1 accuracy, Top-5 accuracy and BOPs of the compressed ResNet-50 are 75.9\%, 92.8\% and 70.4G, respectively.}
    }
    \label{fig:config_r50_main}
    \vspace{-0.1in}
\end{figure}

\subsection{Effect of the bit sharing scheme}
\label{sec:effect_bit_sharing}
To investigate the effect of the bit sharing scheme, we apply \methodshortname to quantize ResNet-20 and ResNet-56 with  and without the bit sharing scheme and report the results on CIFAR-100. 
\jing{Here, \methodshortname without the bit sharing denotes that we compress the models with the multi-path scheme~\cite{wu2018mixed}.}
We report the Top-1/Top-5 accuracy and BOPs in Table~\ref{table:effect_bit_sharing}. We also present the search cost and GPU memory footprints measured on a GPU device (NVIDIA TITAN Xp). From the results, \methodshortname with the bit sharing scheme obtains more compact models and consistently outperforms the ones without the bit sharing scheme while significantly reducing the search cost and GPU memory footprints. 
For example, ResNet-56 with the bit sharing scheme outperforms the counterpart by 0.2\% in the Top-1 accuracy and achieves 3.6$\times$ reduction on GPU memory and 6.7$\times$ acceleration during training.

\subsection{Effect of the one-stage compression}
\label{sec:effect_one_stage_compression}
To investigate the effect of the one-stage compression scheme, we perform model compression with both the one-stage and the two-stage compression schemes on ResNet-56. Specifically, the one-stage compression scheme means that we perform filter pruning and quantization jointly. \jingr{The two-stage compression scheme is that we perform filter pruning and network quantization in a separate stage. For convenience, we denote the two-stage compression scheme as A$\rightarrow$B,
where A$\rightarrow$B denotes that we first perform A and then conduct B. We consider both SBS-Q $\rightarrow$ SBS-P and SBS-P $\rightarrow$ SBS-Q for comparisons.} \jingr{For the two-stage compression method, we have conducted extensive experiments to find a trade-off between pruning and quantization that leads to high accuracy in the final compressed model following~\cite{wang2020diff}.} From Table~\ref{table:effect_one_stage_compression},
the resulting model obtained by the one-stage scheme with less computation cost outperforms the two-stage counterparts. For example, SBS with 1064.1M BOPs outperforms SBS-P $\rightarrow$ SBS-Q by 0.2\% on the Top-1 accuracy. These results show the superiority of performing pruning and quantization jointly.

\subsection{Effect of the alternating training scheme}
\label{sec:effect_algernative_training}
\liu{To investigate the effect of the alternating training scheme introduced in Algorithm~\ref{alg:lbs},}
we apply \methodshortname to compress ResNet-56 with a joint training scheme and  an alternating training scheme on CIFAR-100. 
\jing{Here, the joint training scheme denotes that we train the binary gates of weights and activations jointly. The alternating training scheme indicates that we train the binary gates of weights and activations in an alternating way, as mentioned in Section~\ref{sec:objective_lbs}.}
From the results of Table~\ref{table:effect_alternating_training}, the model trained with the alternating scheme achieves better performance than those of the joint scheme while consuming lower computational overhead, which demonstrates the effectiveness of the proposed alternating training scheme.

\begin{table}[!tb]
\renewcommand{\arraystretch}{1.3}
\caption{Effect of the one-stage compression. We report the results on CIFAR-100. A$\rightarrow$B denotes that we first perform A and then conduct B.
}
\vspace{-0.1in}
\begin{center}
\scalebox{0.82}
{
\begin{tabular}{cccccc}
\toprule
Network  & Method & Top-1 Acc. (\%) & Top-5 Acc. (\%) & BOPs (M) \\ 
\midrule
\multirow{3}{*}{ResNet-56} 
& \methodshortname-P $\rightarrow$ \methodshortname-Q & 70.7$\pm$0.2 & 91.3$\pm$0.2 & 1092.4 \\
& \methodshortname-Q $\rightarrow$ \methodshortname-P & 70.6$\pm$0.3 & 91.4$\pm$0.3 & 1067.1\\
&  \methodshortname & \textbf{70.9$\pm$0.3} & \textbf{91.4$\pm$0.2} & \textbf{1064.1} \\
\bottomrule
\end{tabular}
}
\end{center}
\label{table:effect_one_stage_compression}
\vspace{-0.1in}
\end{table}

\begin{table}[!tb]
\renewcommand{\arraystretch}{1.3}
\caption{Effect of the alternating training scheme. We report the results of ResNet-56 on CIFAR-100.
}
\vspace{-0.2in}
\begin{center}
\scalebox{0.82}
{
\begin{tabular}{cccccc}
\toprule
Network  & Method & Top-1 Acc. (\%) & Top-5 Acc. (\%) & BOPs (M) \\ 
\midrule
\multirow{2}{*}{ResNet-56} 
& Joint training  & 71.3$\pm$0.2 & 91.6$\pm$0.3 & 1942.4\\
&  Alternating training & \textbf{71.6$\pm$0.1} & \textbf{91.8$\pm$0.4} & \textbf{1918.8} \\
\bottomrule
\end{tabular}
}
\end{center}
\label{table:effect_alternating_training}
\vspace{-0.1in}
\end{table}

\begin{table}[!tb]
    \renewcommand{\arraystretch}{1.3}
    \caption{Performance comparisons with different group sizes $B$ on CIFAR-100.}
    \vspace{-0.1in}
    \begin{center}
    \scalebox{0.8}
    {
    \begin{tabular}{ccccccc}
    \toprule
    Network  & $B$ & Top-1 Acc. (\%) & Top-5 Acc. (\%) & BOPs (M) \\ 
    \midrule
    \multirow{4}{*}{ResNet-20} & 1 & 67.7$\pm$0.2 & 90.5$\pm$0.1 & 648.3 \\
    & 2 & 68.0$\pm$0.1 & 90.4$\pm$0.1 & 631.2 \\
    & 4 & \textbf{68.1$\pm$0.3} & \textbf{90.6$\pm$0.2} & \textbf{630.6} \\
    & 8 & 67.9$\pm$0.1 & 90.4$\pm$0.1 & 631.1 \\
    \midrule
    \multirow{4}{*}{ResNet-56} & 1 & 71.3$\pm$0.1 & 91.6$\pm$0.2 & 1945.3 \\
    & 2 & 71.3$\pm$0.1 & 91.6$\pm$0.1 & 1949.5 \\
    & 4 & \textbf{71.6$\pm$0.1} & \textbf{91.8$\pm$0.4} & \textbf{1918.8} \\
    & 8 & 71.5$\pm$0.2 & $91.4\pm0.2$ & 1956.5 \\
    \bottomrule
    \end{tabular}
    }
    \end{center}
    \label{table:effect_B}
    \vspace{-0.1in}
\end{table}

\subsection{Effect of different group sizes}
\label{sec:effect_group_size}
\jingr{To investigate the effect of different group sizes $B$, we apply SBS to compress ResNet-20 and ResNet-56 with different $B \in \{ 1, 2, 4, 8 \}$ and show the results in Table~\ref{table:effect_B}. From the table, the performance of the compressed models first improves and then degrades with the increase of $B$. For example, the compressed ResNet-56 with $B=4$ outperforms those of $B=1$ by 0.3\% in terms of the Top-1 accuracy. In fact, a large $B$ leads to a small search space. With limited computing resources, we are able to find good configurations with high probability and thus improve performance. In contrast, a too large $B$ results in extremely small search space, which may ignore many good compression configurations and thus limits the performance. Since our SBS achieves the best performance with $B=4$, we use it by default for the experiments on CIFAR-100.}

\begin{table}[!tb]
    \renewcommand{\arraystretch}{1.3}
    \caption{Performance comparisons of different methods on CIFAR-100. $^{*}$ denotes our SBS with the candidate bitwidths of $\{ 3, 6, 12 \}$.
    }
    \vspace{-0.1in}
    \begin{center}
    \scalebox{0.8}
    {
    \begin{tabular}{ccccccc}
    \toprule
    Network  & Method & Top-1 Acc. (\%) & Top-5 Acc. (\%) & BOPs (M) \\ 
    \midrule
    \multirow{3}{*}{ResNet-20} 
    & Full-precision & 67.5 & 90.8 & 41798.6 \\
    & 6-bit precision & 68.3$\pm$0.1 & 90.7$\pm$0.2 & 1482.1 \\
    & SBS$^{*}$ & \textbf{68.4$\pm$0.2} & \textbf{90.8$\pm$0.2} & \textbf{1481.2}\\
    \cline{1-6}
    \multirow{3}{*}{ResNet-56} 
    & Full-precision & 71.7 & 92.2 & 128771.7 \\
    & 6-bit precision & 71.7$\pm$0.1 & 91.7$\pm$0.4 & 4539.7 \\
    & SBS$^{*}$ & \textbf{71.8$\pm$0.2} & \textbf{91.8$\pm$0.3} & \textbf{4504.3} \\
    \bottomrule
    \end{tabular}
    }
    \end{center}
    \label{table:different_bitwidths_cifar}
    \vspace{-0.2in}
\end{table}

\subsection{Effect of SBS with non-power-of-two bitwidths.}
\label{sec:comparison_power_of_two_non_power_of_two}
\jingr{To investigate the effect of our \methodshortname on non-power-of-two bitwidths, we apply SBS to compress ResNet-20/56 on CIFAR-100 and ResNet-18/50 on ImageNet with the non-power-of-two bitwidths. As shown in Theorem~\ref{prop:bit_decompistion}, if we set $b_1$ to 3 and $\gamma_j$ to 2, we have the quantization decomposition for an example of the non-power-of-two bitwidths (\ie, $b_j \in \{ 3, 6, 12, 24, \cdots \}$). Similar to the power-of-two bitwidths, we only consider those bitwidths that are not greater than 12-bit since the normalized quantization error change is small when the bitwidth is greater than 12 according to Corollary~\ref{coll:bit_error}.
We denote SBS with the candidate bitwidths of $\{ 3,6,12 \}$ as SBS$^*$ for convenience. From Tables~\ref{table:different_bitwidths_cifar} and~\ref{table:different_bitwidths_imagenet}, SBS$^*$ obtains fewer BOPs while achieving better performance than the fixed-precision counterparts. For example, for ResNet-50, SBS$^{*}$ outperforms the 6-bit conterpart by 0.2\% on the Top-1 accuracy. 
These results show the effectiveness of SBS with the non-power-of-two bitwidths.
Moreover, the non-power-of-two bitwidths are not hardware-friendly due to the bit wasting in values packing (See Appendix~\ref{sec:hardware_friendly_decomposition}), which seriously degrades the hardware utilization. Therefore, we only consider the power-of-two bitwidths in our experiments.}

\begin{table}[!tb]
    \renewcommand{\arraystretch}{1.3}
    \caption{Performance comparisons of different methods on ImageNet. $^{*}$ denotes our SBS with the candidate bitwidths of $\{ 3, 6, 12 \}$.
    }
    \vspace{-0.1in}
    \begin{center}
    \scalebox{0.8}
    {
    \begin{tabular}{cccccccc}
    \toprule
    Network  & Method & Top-1 Acc. (\%) & Top-5 Acc. (\%) & BOPs (G) \\ 
    \midrule
    \multirow{3}{*}{ResNet-18} & Full-precision & 69.8 & 89.1 & 1857.6 \\
    & 6-bit precision & 69.9 & 89.3 & 68.6 \\
    & SBS$^{*}$ & \textbf{70.1} & \textbf{89.4} & \textbf{67.7} \\
    \cline{1-6}
    \multirow{3}{*}{ResNet-50} & Full-precision & 76.8 & 93.3 & 4187.3 \\
    & 6-bit precision & 76.1 & 93.0 & 150.6 \\
    & SBS$^{*}$ & \textbf{76.3} & \textbf{93.1} & \textbf{150.0} \\
    \bottomrule
    \end{tabular}
    }
    \end{center}
    \label{table:different_bitwidths_imagenet}
    \vspace{-0.1in}
\end{table}

\begin{table}[!t]
    \renewcommand{\arraystretch}{1.3}
    \caption{Performance comparisons of different methods on CIFAR-100. All the compressed models are trained from scratch.
    }
    \vspace{-0.05in}
    \centering
    \scalebox{0.8}
    {
    \begin{tabular}{ccccccc}
    \toprule
    Network  & Method & Top-1 Acc. (\%) & Top-5 Acc. (\%) & BOPs (M) \\
    \midrule
    \multirow{3}{*}{ResNet-20} & Full-precision & 67.5 & 90.8 & 41798.6 \\
    & 4-bit precision & 65.7$\pm$0.2 & 89.4$\pm$0.3 & 674.6 \\
    & \methodshortname (Ours) & \textbf{65.9$\pm$0.2} & \textbf{89.6$\pm$0.2} &\textbf{657.7} \\
    \midrule
    \multirow{3}{*}{ResNet-56} & Full-precision & 71.7 & 92.2 & 128771.7  \\
    & 4-bit precision  & 67.4$\pm$0.1 & 89.9$\pm$0.3 & 2033.6\\
    & \methodshortname (Ours) & \textbf{67.7$\pm$0.4} & \textbf{90.2$\pm$0.2} & \textbf{1981.3}  \\
    \bottomrule
    \end{tabular}
    }
    \label{table:from_scratch_on_CIFAR_100}
    \vspace{-0.1in}
\end{table}

\subsection{Training from scratch \vs~fine-tuning scheme}
\label{sec:fine_tuning}
\jingr{To investigate the effect of \methodshortname with training from scratch scheme, we apply \methodshortname to compress ResNet-20 and ResNet-56 on CIFAR-100. The experimental settings are the same as Section~\ref{sec:implementation_details} except that we do not use any pre-trained model. From Table~\ref{table:from_scratch_on_CIFAR_100}, the models obtained by our \methodshortname surpass the full-precision counterparts. Moreover, the performance improvement brought from \methodshortname is smaller than those that are from fine-tuning. For example, for ResNet-56, the improvement of SBS over the fixed-precision quantization is 0.3\% while the fine-tuning counterpart is 0.7\% in Table~\ref{table:results_on_CIFAR_100}. As the model are not well-trained, the searched configurations may not be accurate, which hinders the effect of our method.}

\section{Conclusion and Future Work}
\jing{In this paper, we have proposed an automatic loss-aware model compression method called \methodfullname (\methodshortname) for pruning and quantization jointly.} The proposed \methodshortname introduces a novel single-path bit sharing model to encode all bitwidths in the search space, where 
a quantized representation will be decomposed into the sum of the lowest bitwidth representation and a series of re-assignment offsets.
Based on this, we have further introduced learnable binary gates to encode the choice of different compression configurations. By jointly training the binary gates and network parameters, we are able to make a trade-off between pruning and quantization and automatically learn the configuration of each layer. \jing{Experiments on CIFAR-100 and ImageNet have shown that SBS is able to achieve significant computation cost and memory footprints reduction while preserving performance.} 
\liu{In the future, we plan to combine our proposed \methodshortname with other methods, \eg, neural architecture search,  to find a more compact model with better performance.}

\ifCLASSOPTIONcompsoc
  \section*{Acknowledgments}
\else
  \section*{Acknowledgment}
\fi
\jingrs{This work was partially supported by the Key-Area Research and Development Program of Guangdong Province (2019B010155002, 2018B010107001), Key Realm R\&D Program of Guangzhou 202007030007, National Natural Science Foundation of China (NSFC) 62072190, Program for Guangdong Introducing Innovative and Enterpreneurial Teams 2017ZT07X183, and National Key R\&D Program of China 2022ZD0118700.}

\ifCLASSOPTIONcaptionsoff
  \newpage
\fi
\bibliographystyle{abbrv}
{
	\bibliography{reference}
}

\begin{IEEEbiography}[{\includegraphics[width=1in,height=1.25in,clip,keepaspectratio]{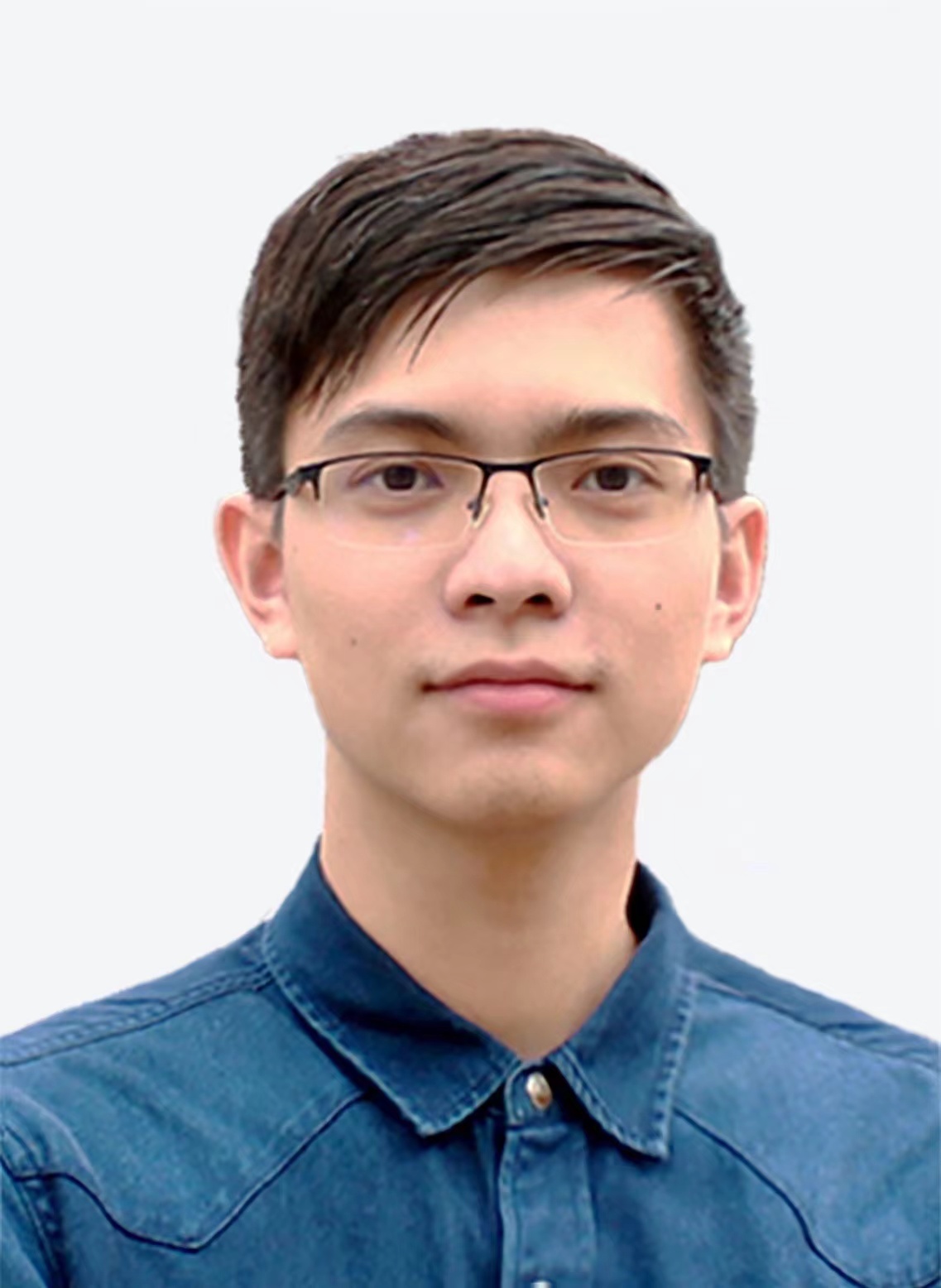}}]{Jing Liu}
is a Ph.D. student in the Faculty of Information Technology, Monash University Clayton Campus, Australia. He received his Bachelor Degree in 2017 and Master Degree in 2020, both from the School of Software Engineering at South China University of Technology, China.
His research interests include computer vision, model compression and acceleration.
\end{IEEEbiography}
\vspace{-1em}

\begin{IEEEbiography}[{\includegraphics[width=1in,height=1.25in,clip,keepaspectratio]{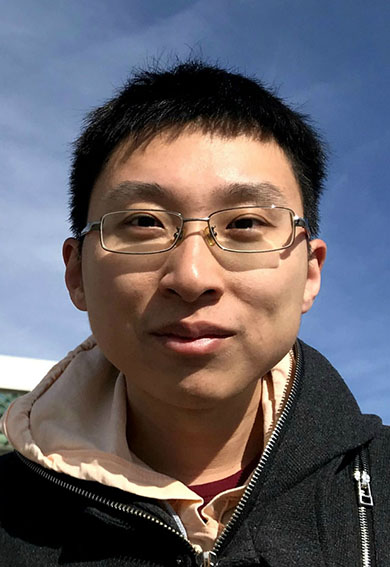}}]{Bohan Zhuang} is an assistant professor with the Faculty of Information Technology, Monash University Clayton Campus, Australia.
He received his Bachelor degree in Electronic and Information Engineering in 2014 from Dalian University of Technology, China. He has been with the School of Computer Science, University of Adelaide, Australia, where he received the Ph.D. degree in 2018 and worked as a Senior Research Associate from 2018-2020. His research interests include computer vision and machine learning. 
\end{IEEEbiography}
\vspace{-1em}

\begin{IEEEbiography}[{\includegraphics[width=1in,height=1.25in,clip,keepaspectratio]{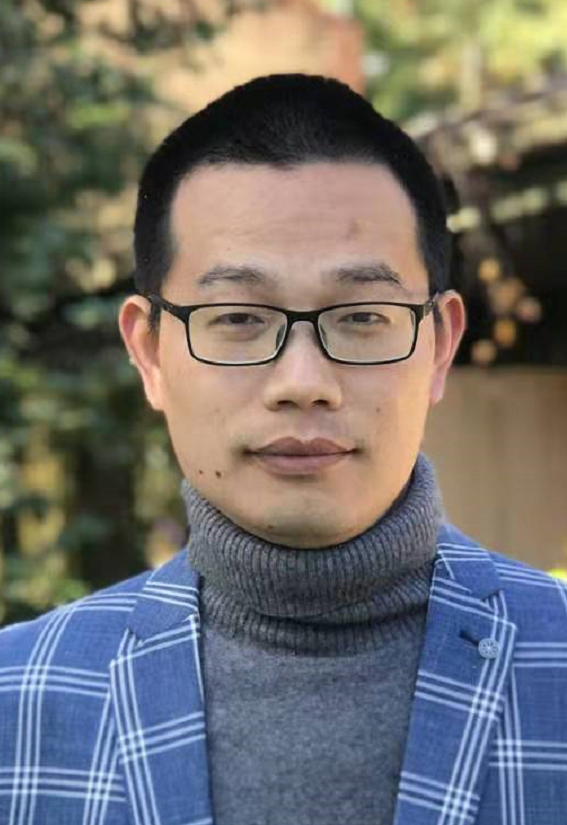}}]{Peng Chen}
was a postdoc researcher at The University of Adelaide. He received his B.S. and PhD degrees from the University of Science and Technology of China. His research interests include computer vision and machine learning algorithms as well as model compression technologies.
\end{IEEEbiography}
\vspace{-1em}

\begin{IEEEbiography}
[{\includegraphics[width=1in,height=1.25in,clip,keepaspectratio]{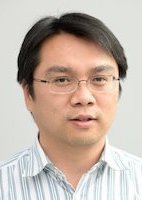}}]
{Chunhua Shen} is a Professor at the College of Computer Science and Technology, Zhejiang University.
\end{IEEEbiography}
\vspace{-1em}

\begin{IEEEbiography}[{\includegraphics[width=1in,height=1.25in,clip,keepaspectratio]{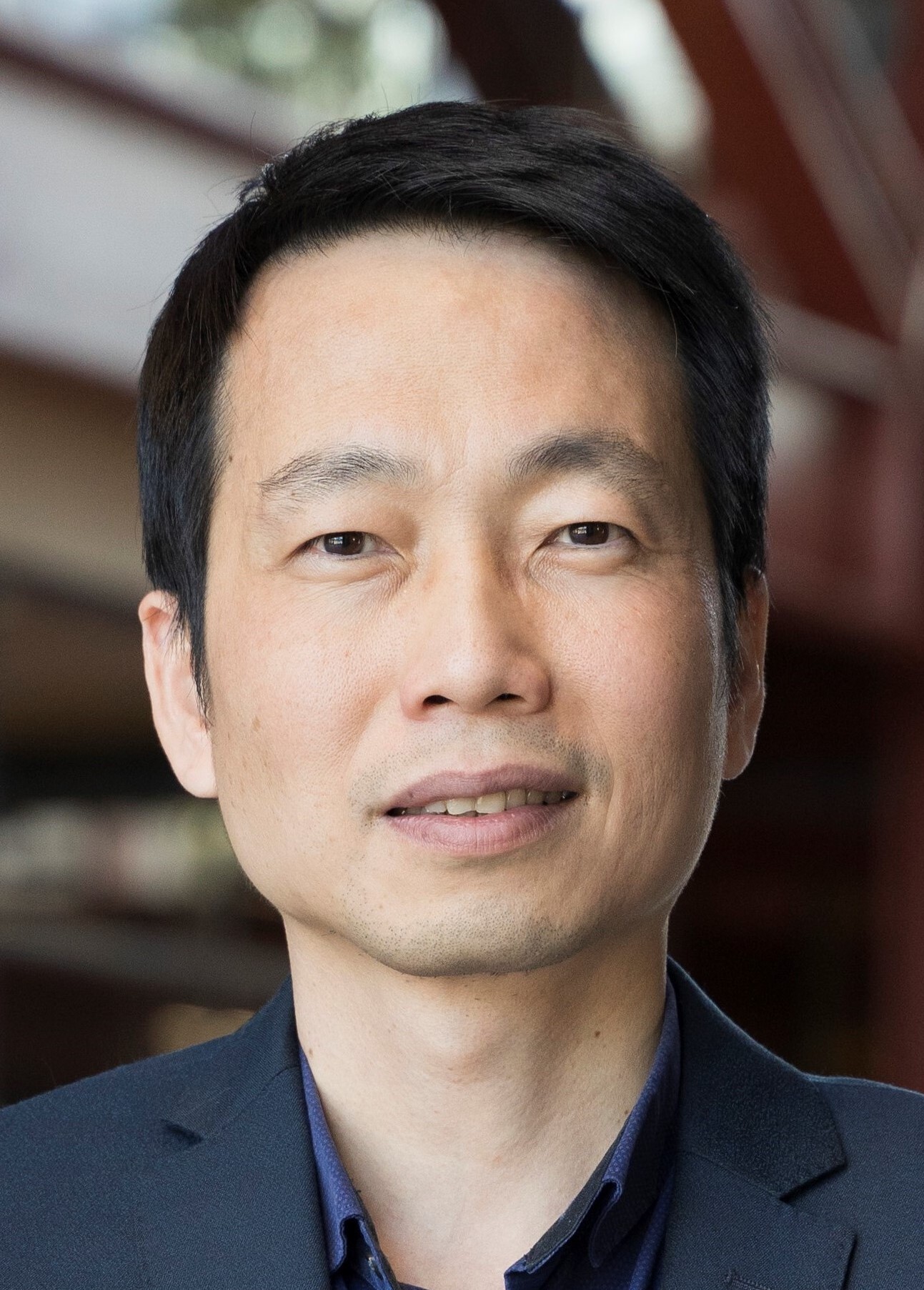}}]{Jianfei Cai}(S'98-M'02-SM'07-F’21) received his PhD degree from the University of Missouri-Columbia. He is currently a Professor and serves as the Head of the Data Science \& AI Department at Faculty of IT, Monash University, Australia. Before that, he had served as Head of Visual and Interactive Computing Division and Head of Computer Communications Division in Nanyang Technological University (NTU). His major research interests include computer vision, multimedia and visual computing. He has successfully trained 30+ PhD students with three getting NTU SCSE Outstanding PhD thesis award. He is a co-recipient of paper awards in ACCV, ICCM, IEEE ICIP and MMSP. He serves or has served as an Associate Editor for IJCV, IEEE T-IP, T-MM, and T-CSVT as well as serving as Area Chair for CVPR, ICCV, ECCV, IJCAI, ACM Multimedia, ICME and ICIP. He was the Chair of IEEE CAS VSPC-TC during 2016-2018. He has also served as the leading TPC Chair for IEEE ICME 2012 will be the leading general chair for ACM Multimedia 2024. He is a Fellow of IEEE.
\end{IEEEbiography}
\vspace{-1em}

\begin{IEEEbiography}[{\includegraphics[width=1in,height=1.25in,clip,keepaspectratio]{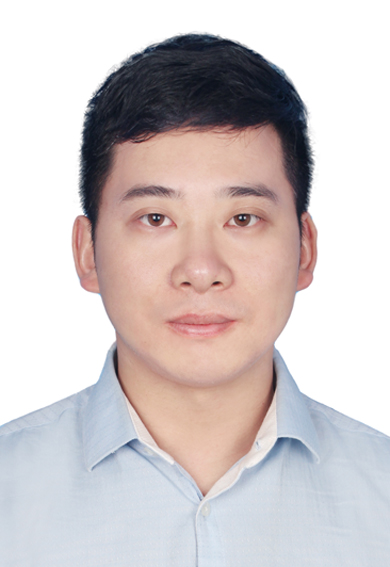}}]{Mingkui Tan}
	is currently a professor with the School of Software Engineering at South China University of Technology.
	He received his Bachelor Degree in Environmental Science and Engineering in 2006 and Master degree in Control Science and Engineering in 2009, both from Hunan University in Changsha, China. He received the Ph.D. degree in Computer Science from Nanyang Technological University, Singapore, in 2014.  From 2014-2016, he worked as a Senior Research Associate on computer vision in the School of Computer Science, University of Adelaide, Australia.
	His research interests include  machine learning,  sparse analysis, deep learning and large-scale optimization.
\end{IEEEbiography}

\appendices
\onecolumn

\begin{LARGE}
~~~\vspace{1pt}
\begin{center}
    \bf Supplementary Materials for ``\mytitle''
\end{center}
\end{LARGE}
\vspace{2pt}

\renewcommand\thefigure{\Roman{figure}}
\renewcommand{\theequation}{\Alph{equation}}
\renewcommand{\thetable}{\Roman{table}}
\setcounter{figure}{0}
\setcounter{equation}{0}
\setcounter{table}{0}
\setcounter{thm}{0}
\setcounter{prop}{0}
\setcounter{coll}{0}

In the supplementary, we provide the detailed proof of Theorem 1, Corollary 1, Proposition 1, more details and more experimental results of the proposed \methodshortname. We organize the supplementary material as follows.

\begin{itemize}[leftmargin=*]
    \item In Section~\ref{sec:proof_theorem_1}, we provide the proof for Theorem 1.
    \item In Section~\ref{sec:proof_corollary_1}, we give the proof for Corollary 1.
    \item In Section~\ref{sec:proof_proposition_2}, we offer more details about the proof of Proposition 1.
    \item In Section~\ref{sec:hardware_friendly_decomposition}, we provide more explanations on the hardware-friendly decomposition.
    \item In Section~\ref{sec:search_space}, we give more details about the search space of \methodshortname.
    \item In Section~\ref{sec:quantization_configurations}, we give more details about the quantization configurations of different methods.
    \item In Section~\ref{sec:results_mobilenetv3}, we offer more results in terms of MobileNetV3 on CIFAR-100.
    \item In Section~\ref{sec:detailed_structure_compressed_network}, we report the searched compression configurations of the compressed models.
\end{itemize}

\section{Proof of Theorem 1}\label{sec:proof_theorem_1}
To prove Theorem 1, we first provide a lemma of quantization formulas as follows.
\begin{lemma}
\label{lem:Recursive equation}
    Let $z \in \left[ 0, 1 \right]$ be a normalized full-precision input, and $\{b_j\}_{j=1}^{K} $ be a sequence of candidate bitwidths. If $b_j$ is an integer multiple of $b_{j-1}$, \ie, $b_{j} \small{=} \gamma_j b_{j-1}(j>1)$, where $\gamma_j\in \sZ^+ \backslash \{1\}$ is a multiplier, then the quantized $z_{b_{j+1}}$ can  be  decomposed into the quantized $z_{b_{j}}$ and quantized value re-assignment $r_{b_{j+1}}$
        \begin{equation}
        \begin{aligned}
        \label{eq:bit_decompistion1}
        z_{b_{j+1}} = z_{b_{j}} &+ r_{b_{j+1}},\\ \mathrm{where}~r_{b_{j+1}} &= D(z - z_{b_{j}}, s_{b_{j+1}}), ~\\
        z_{b_j} &= D(z , s_{b_j}), ~\\
        s_{b_j} &=\frac{1}{2^{b_j} - 1}.
        \end{aligned}
        \end{equation}
\end{lemma}

\begin{proof}
We construct two sequences 
$\{A_m\}$ and $\{B_n\}$ for the $b_j$-quantized value and the $b_{j+1}$-quantized value\\
\begin{equation}
    \begin{aligned}
    \left\{0, \frac{1}{2^{b_j}-1}, \frac{2}{2^{b_j}-1}, \ldots, \frac{2^{b_j}-2}{2^{b_j}-1}, 1  \right\},
    \end{aligned}
\end{equation}
\begin{equation}
    \begin{aligned}
    \left\{0, \frac{1}{2^{b_{j+1}}-1}, \frac{2}{2^{b_{j+1}}-1}, \ldots, \frac{2^{b_{j+1}}-2}{2^{b_{j+1}}-1}, 1  \right\}.
    \end{aligned}
\end{equation}
First, we can obtain each value in $\{A_m\}$ is in $\{B_n\}$, \ie,  $\{A_m\} \subset\{B_n\} $ since $2^{b_{j+1}} \small{-} 1 \small{=} (2^ {b_j})^{\gamma_{j+1}} \small{-} 1$ is divisible by $2^{b_j} \small{-} 1$. Then we rewrite the sequence $\{B_n\}$ as:\\
\begin{equation}
\label{eq:sequence_Bn}
    \begin{aligned}
    \Bigg\{0, \frac{1}{2^{b_{j+1}}-1}, \ldots, \underbrace{\frac{1}{2^{b_j}-1}}_{A_2}, \ldots, \underbrace{\frac{t-1}{2^{b_{j+1}}-1}}_{B_t}, \underbrace{\frac{t}{2^{b_{j+1}}-1}}_{B_{t+1}}, \dots,
    \underbrace{\frac{2}{2^{b_{j}}-1}}_{A_3}, \ldots, \frac{2^{b_{j+1}}-2}{2^{b_{j+1}}-1}, 1 \Bigg\}.
    \end{aligned}
\end{equation}
Note that $A_1, A_2, \ldots \in \{B_n\}$, thus we only focus on the sequence $B_n$. For an attribute $z \in \left[0,1\right]$, it surely falls in some interval $\left[A_i, A_{i+1} \right]$. Without loss of generality, we assume $z \in \left[A_2, A_3\right] $ and $z \in \left[B_t, B_{t+1}\right] $ (see sequence \ref{eq:sequence_Bn}).

Next, we will discuss the following cases according to the position of $z$ in $\left[A_2, A_3\right]$:\\
1) If $z \in [A_2, \frac{A_2 + A_3}{2}]$, based on the definition of $z_{b_{j}} = D(z , s_{b_j})$, we have $z_{b_{j}} = A_2$. 
Moreover, if $z \in [B_t, \frac{B_t + B_{t+1}}{2}]$, then $z_{b_{j+1}} = B_t$ and $r_{b_{j+1}} = D(z - A_2, s_{b_{j+1}}) = B_t - A_2$. Thus, we get $z_{b_{j+1}} = z_{b_{j}} + r_{b_{j+1}}$; 
otherwise if $z \in (\frac{B_t + B_{t+1}}{2}, B_{t+1}]$, then $z_{b_{j+1}} = B_{t+1}$ and $r_{b_{j+1}}= B_{t+1} - A_2$. Therefore, we have the same conclusion.
\\
2) If $z \in (\frac{A_2 + A_3}{2}, A_3]$, similar to the first case, we have $z_{b_{j}} = A_3$. Moreover, if $z \in [B_t, \frac{B_t + B_{t+1}}{2}]$, then $z_{b_{j+1}} = B_t$ and $r_{b_{j+1}} = D(z - A_3, s_{b_{j+1}}) = -(A_3 - B_t)$. Thus we have $z_{b_{j+1}} = z_{b_{j}} + r_{b_{j+1}}$; otherwise if  $z \in (\frac{B_t + B_{t+1}}{2}, B_{t+1}]$, then $z_{b_{j+1}} = B_{t+1}$ and $r_{b_{j+1}} = -(A_3 - B_{t+1})$. Hence we still obtain the same conclusion.

\end{proof}

\begin{thm}
    Let $z \in \left[ 0, 1 \right]$ be a normalized full-precision input, and $\{b_j\}_{j=1}^{K} $ be a sequence of candidate bitwidths. If $b_j$ is an integer multiple of $b_{j-1}$, \ie, $b_{j} \small{=} \gamma_j b_{j-1}(j>1)$, where $\gamma_j\in \sZ^+ \backslash \{1\}$ is a multiplier, then the following quantized approximation $z_{b_K}$ can be decomposed as: 
        \begin{equation}
        \begin{aligned}
        \label{eq:bit_decompistion2}
        &z_{b_K} = z_{b_1} + \sum_{j=2}^{K }r_{b_j}, ~\\ &\mathrm{where}~r_{b_{j}} = D(z - z_{b_{j-1}}, s_{b_{j}}), ~\\
        &~~~~~~~~~~~z_{b_j} = D(z , s_{b_j}), ~\\
        &~~~~~~~~~~~s_{b_j} =\frac{1}{2^{b_j} - 1}.
        \end{aligned}
        \end{equation}
\end{thm}
\begin{proof}
By summing the equations $z_{b_j} \small{=} z_{b_{j-1}} + r_{b_j}$ in lemma \ref{lem:Recursive equation} from $j \small{=} 2$ to $K$, we complete the results.
\end{proof} 

\section{Proof of Corollary 1}\label{sec:proof_corollary_1}
\begin{coll}
\textbf{\emph{(Normalized  Quantization Error Bound) }} 
\label{coll:bit_error1}
Given $\bz \in [0, 1]^d$ being a normalized full-precision vector, $\bz_{b_K} $ being its quantized vector with bitwidth $b_K$, where $d$ is the cardinality of $\bz$. Let $\epsilon_K = \frac{\|\bz - \bz_{b_K}\|_1}{\|\bz\|_1}$ be the normalized quantization error, then the following error bound w.r.t. $K$ holds
\begin{equation}
    \begin{aligned}
    |\epsilon_K \small{-} \epsilon_{K+1}|\leq \frac{C}{2^{b_{K}}\small{-}1},
   \end{aligned}
\end{equation}
where $~C \small{=} \frac{d}{\|\bz\|_1}$ is a constant.
\end{coll}
\begin{proof}
Let $\bz = [z^1, z^2, \ldots, z^d]$ and $\bz_{b_K} = [z_{b_K}^1, z_{b_K}^2, \ldots, z_{b_K}^d]$, we rewrite the normalized quantization error as:
\begin{equation}
    \epsilon_K = \frac{\sum_{i=1}^d |z^i-z_{b_K}^i|}{\|\bz\|_1}.
\end{equation}
Then, we get the following result
\begin{equation}
\begin{aligned}
    |\epsilon_K \small{-} \epsilon_{K+1}| \small{=} \frac{1}{\|\bz\|_1} \sum_{i=1}^d \left| |z^i-z_{b_K}^i| \small{-} |z^i-z_{b_{K+1}}^i| \right| &\small{\leq} \frac{1}{\|\bz\|_1} \sum_{i=1}^d \left| z_{b_K}^i \small{-} z_{b_{K+1}}^i \right|\small{=} \frac{1}{\|\bz\|_1} \sum_{i=1}^d \left| r_{b_{K+1}}^i \right|\\
    &\small{\leq} \frac{1}{\|\bz\|_1} \sum_{i=1}^d \left| s_{b_{K}}\right| \small{=} \frac{d}{\|\bz\|_1 (2^{b_{K}}\small{-}1)} = \frac{C}{2^{b_K} - 1}.
\end{aligned}
\end{equation}
\end{proof}

To empirically demonstrate Corollary~\ref{coll:bit_error}, we perform quantization on a random toy data $\bz \in [0, 1]^{100}$ and show the normalized quantization error change $|\epsilon_K \small{-} \epsilon_{K+1}|$ and error bound $\frac{C}{2^{b_{K}}\small{-}1}$ in Fig.~\ref{fig:quantization_error_supp}. From the results, the normalized quantization error change decreases quickly as the bitwidth increases and is bounded by $C/(2^{b_{K}}\small{-}1)$.

\begin{figure}[!t]
	\centering
	\includegraphics[width = 0.38\linewidth]{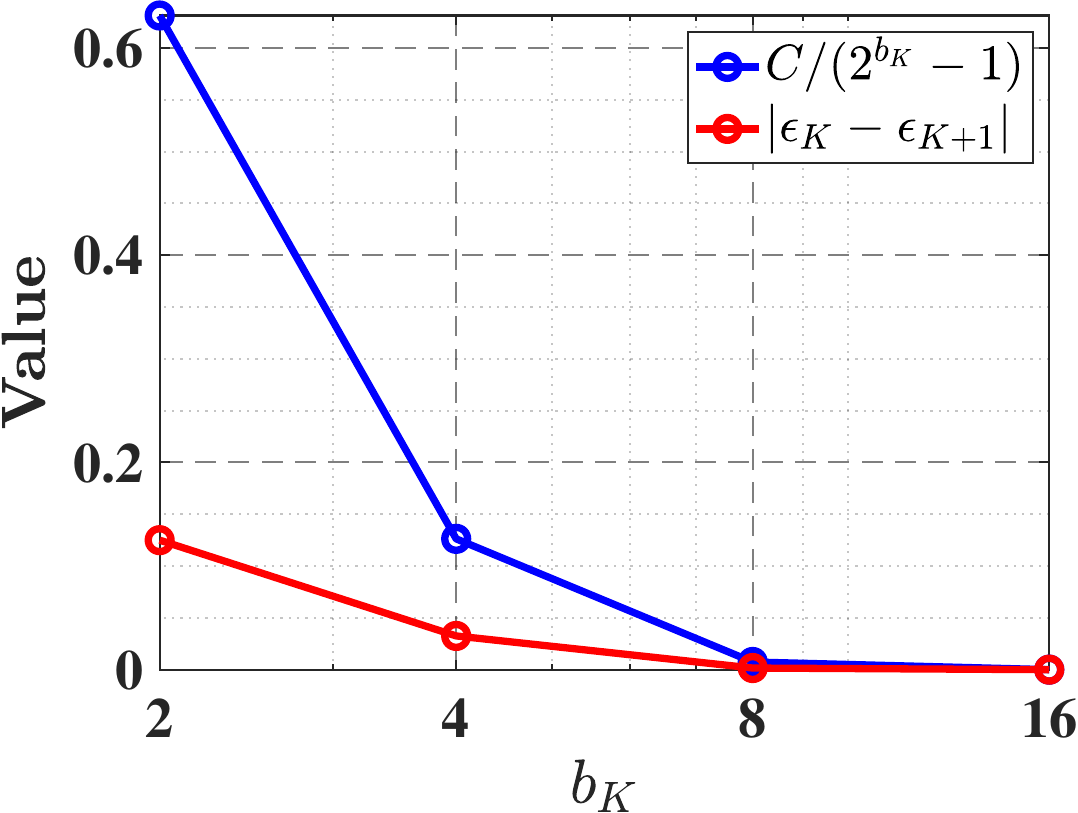}
	\caption{\jing{The normalized quantization error change and error bound vs. bitwidth. The normalized quantization error change (red line) is bounded by $C/(2^{b_{K}}\small{-}1)$ (blue line).}
	}
	\label{fig:quantization_error_supp}
\end{figure}

\section{Proof of Proposition 1}
\label{sec:proof_proposition_2}
To prove Proposition 1, we derive a lemma of normalized quantization error as follows.
\begin{lemma}
\label{lemma2:error decrese}
Given $\bz \in [0,1]^d$ being a normalized full-precision vector and $\bz_{b_K}$ being its quantized vector with bitwidth $b_K$, where $d$ is the cardinality of $\bz$, let $\epsilon_K \small{=} \frac{\|\bz -\bz_{b_K}\|_1}{\|\bz\|_1}$ be  the
 normalized quantization error, then $\epsilon_K$ decreases with the increase of K.
 \end{lemma}
\begin{proof}
Let $\bz = [z^1, z^2, \ldots, z^d]$ and $\bz_{b_K} = [z_{b_K}^1, z_{b_K}^2, \ldots, z_{b_K}^d]$, we rewrite the normalized quantization error as:
\begin{equation}
\label{revised ever}
       \epsilon_K = \frac{\sum_{i=1}^d |z^i-z_{b_K}^i|}{\|\bz\|_1}.
\end{equation}
Since $z_{b_K}$ is the only variable, we focus on one term $z-z_{b_K}$ of the numerator  in Eq. (\ref{revised ever}). Next, we will prove $\vert z-z_{b_{j}} \vert \geq \vert z-z_{b_{j+1}} \vert$. Following the sequence (\ref{eq:sequence_Bn}) in Lemma \ref{lem:Recursive equation} and without loss of generality, we assume $z \in \left[B_t, B_{t+1}\right] \subset \left[A_2, A_3\right] $, then we have the following conclusions:\\
1) If $A_2 = B_t < B_{t+1} < A_3$, then $\vert z-z_{b_{j}} \vert = \vert z-A_2 \vert \geq \vert z-z_{b_{j+1}} \vert$, where $\vert z-A_2 \vert {=} \vert z-z_{b_{j+1}} \vert$ happens only when $z \in [B_t, \frac{B_t+B_{t+1}}{2}]$ ;\\
2) If $A_2 < B_t < B_{t+1} < A_3$, then $\vert z-z_{b_{j}} \vert \geq B_{t+1}-B_t > \frac {B_{t+1}-B_t}{2} \geq \vert z-z_{b_{j+1}} \vert$;\\
3) If $A_2 < B_t < B_{t+1} = A_3$, then $\vert z-z_{b_{j}} \vert = \vert z-A_3 \vert \geq \vert z-z_{b_{j+1}} \vert$, where $\vert z-A_3 \vert {=} \vert z-z_{b_{j+1}} \vert$ happens only when $z \in [\frac{B_t+B_{t+1}}{2}, B_{t+1}]$.

Therefore, we have $\vert z-z_{b_{j}} \vert \geq \vert z-z_{b_{j+1}} \vert$. Thus, summing from $i=1$ to $d$ complete the conclusion.
\end{proof}

\begin{prop}
Consider a linear regression problem $\min \limits_{\bw \in \sR^{d}} \underset{(\bx, y) \sim \gD}  \sE[(y - \bw\bx)^2]$ with data pairs $\{(\bx, y)\}$, where $\bx \in \sR ^d$ is sampled from $\gN(0,\sigma^2\bf I)$ and $y \in \sR $ is its response. Consider using \methodshortname and DNAS to quantize the linear regression weights.
Let $\bw_L^t$ and $\bw_D^t$ be the quantized regression weights of \methodshortname and DNAS at the iteration $t$ of the optimization, respectively.
Then the following equivalence holds during the optimization process
\begin{equation}
    \begin{aligned}
    &\lim_{t \rightarrow \infty} \sE _{(\bx, y)\sim\gD}[(y \small{-} \bw^t_{\text{L}}\, \bx)^2] \small{=} \lim_{t \rightarrow \infty} \sE _{(\bx, y)\sim\gD}[(y \small{-} \bw^t_{\text{D}}\, \bx)^2],\\
    &\mathrm{where}~\bw^t_{\text{L}} \small{=} \bw^t_{L,b_1}\! \!\small{+} g_{b_2}^t \big( \br^t_{b_2} \!\small{+}\! \cdots + \!g_{b_{K-1}}^t \!\big (\br^t_{b_{K-1}} \!\small{+} g_{b_K}^t\! \br^t_{b_K} \! \big)\! \big),\\
    &~~~~~~~~~~~\br^t_{b_{j}} = D(\bw^t_{\text{L}} - \bw^t_{{\text{L}},{b_{j-1}}},s_{b_{j}}),~ j=2,\cdots,K,\\
    &~~~~~~~~~~~\bw^t_{\text{D}} = \sum_{i=1}^K p_i^t\bw_{b_i}^t,~~~\sum_{i=1}^K p^t_i = 1.
     \end{aligned}
\end{equation}
\end{prop}
\begin{proof}
We first rewrite the optimization objective as:
\begin{equation*}
    \begin{aligned}
    (y \small{-} \bw^t\, \bx)^2 \small{=}  (\bw^{\ast}\, \bx \small{-} \bw^t \, \bx \small{+} \Delta_{\bx} )^2 \small{=}  \left((\bw^{\ast}\small{-} \bw^t) \, \bx  )\right)^2 \small{+}  2 \Delta_{\bx} \left \langle \bw^{\ast}\small{-} \bw^t , \bx \right \rangle \small{+} \Delta_{\bx}^2 \small{=} (\br^t \, \bx  )^2 \small{+} 2 \Delta_{\bx} \left \langle \br^t , \bx \right \rangle \small{+} \Delta_{\bx}^2, 
    \end{aligned}
\end{equation*}
where $\bw^*$ is the optimal value of $\bw$, $\Delta_{\bx}$ is the regression error that is a constant, and $\br^t$ denotes the weight quantization error at the $t$-th iteration. Taking expectation, we have following results
\begin{equation*}
    \begin{aligned}
    \sE _{(\bx, y)\sim\gD}[(y \small{-} \bw^t\, \bx)^2] 
    &\small{=} \sE _{\bx\sim \gN(0,\sigma^2\bf I)} (\br^t \, \bx  )^2 \small{+} 2 \Delta_{\bx} \left \langle \br^t , \bx \right \rangle \small{+} \Delta_{\bx}^2\\
    & \small{=} \sE _{\bx\sim \gN(0,\sigma^2\bf I)} (\br^t \, \bx  )^2 \small{+} \Delta_{\bx}^2\\
    & \small{=} \sE _{\bx\sim \gN(0,\sigma^2\bf I)} \left(\sum_{i=1}^d \br_i^t \bx_i \right)^2 \small{+} \Delta_{\bx}^2\\
    & \small{=} \sE _{\bx\sim \gN(0,\sigma^2\bf I)} \sum_{i=1}^d (\br_i^t \bx_i )^2 \small{+} 2\sum \limits_{k \neq j} \br_{k}^t \br_{j}^t \bx_k \bx_j \small{+} \Delta_{\bx}^2\\
    & \small{=}  \sum_{i=1}^d (\br_i^t)^2 \sE _{\bx\sim \gN(0,\sigma^2\bf I)}(\bx^2) \small{+} 2\sum \limits_{k \neq j} \br_{k}^t \br_{j}^t \left(\sE _{\bx\sim \gN(0,\sigma^2\bf I)} (\bx) \right)^2 \small{+} \sE _{\bx\sim \gN(0,\sigma^2\bf I)} \Delta_{\bx}^2\\
    & \small{=}  \sum_{i=1}^d (\br_{i}^t )^2{\sigma}^2  \small{+}  \sE _{\bx\sim \gN(0,\sigma^2\bf I)}\Delta_{\bx}^2,\\
    \end{aligned}
\end{equation*}
where we use assumption of Gaussian distribution $\sE (\bx) = 0$ and $\sE (\bx^2) = \sigma^2$.

The last equation suggests the training quantization error is only related to the weight quantization error $(\br_{i}^t)$ in an iteration since the term $\sE _{\bx\sim \gN(0,\sigma^2\bf I)}\Delta_{\bx}^2$ is a constant. Followed by Lemma  \ref{lemma2:error decrese}, the quantization error decreases as the bitwidth increases. This means that for each iteration, both the \methodshortname and the DNAS are able to search for the $K$-bit to lower the quantization error by gradient descent.

Therefore, with gradient descent, for attribute $\eta/\sigma^2>0$, there exists some iteration $T$ so that when $t>T$, we have 
\begin{equation*}
    \begin{aligned}
    \left| \sum_{i=1}^d (\tilde{\br}_{i}^t )^2 \small{-}  \sum_{i=1}^d (\hat{\br}_{i}^t )^2\right|\leq \eta/\sigma^2,
    \end{aligned}
\end{equation*}
where $\tilde{\br} = \bw^{\ast}\small{-} \bw_{\text{L}}$  and $\hat{\br} = \bw^{\ast}\small{-} \bw_{\text{D}}$ denote  the weight quantization error with the \methodshortname and the DNAS, respectively. Then multiplying $\sigma^2$ gives that
\begin{equation*}
    \begin{aligned}
    \bigg |    \sE _{(\bx, y)\sim\gD}[(y \small{-} \bw_L^t\, \bx)^2] \small{-}      \sE _{(\bx, y)\sim\gD}[(y \small{-} \bw_D^t\, \bx)^2]\bigg | \leq \eta.
    \end{aligned}
\end{equation*}
Based on the definition of limit, the following result holds
\begin{equation*}
    \begin{aligned}
    \lim_{t \rightarrow \infty} \bigg |    \sE _{(\bx, y)\sim\gD}[(y \small{-} \bw_L^t\, \bx)^2] \small{-}      \sE _{(\bx, y)\sim\gD}[(y \small{-} \bw_D^t\, \bx)^2]\bigg | \small{=} 0.
    \end{aligned}
\end{equation*}
Therefore, we have
\begin{equation*}
    \lim_{t \rightarrow \infty} \sE _{(\bx, y)\sim\gD}[(y \small{-} \bw^t_{\text{L}}\, \bx)^2] \small{=} \lim_{t \rightarrow \infty} \sE _{(\bx, y)\sim\gD}[(y \small{-} \bw^t_{\text{D}}\, \bx)^2].
\end{equation*}
\end{proof}

\section{More details about hardware-friendly decomposition}
\label{sec:hardware_friendly_decomposition}
In this section, we provide more explanations on hardware-friendly decomposition.
As mentioned in Section~3.1, $b_1$ and $\gamma_j$ can be set to arbitrary appropriate integer values (\eg, 2, 3, etc.). By default, we set $b_1$ and $\gamma_j$ to 2 for better hardware utilization. On general-purpose computing devices (\eg, CPU, GPU), $byte$ (8 bits) is the lowest data type for operations. Other data types and ALU registers are all composed of multiple bytes in width. By setting $b_1$ and $\gamma_j$ to 2, 2-bit/ 4-bit/ 8-bit quantization values can be packed into $byte$ (or $short, int, long$) data type without bit wasting. Otherwise, 
if $b_1$ and $\gamma_j$ set to other values,
it is inevitable to have wasted bits when packing mixed-precision quantized tensors on general-purpose devices. For example, a 32-bit $int$ data type can be used to store ten 3-bit quantized values with 2 bits wasted. One might argue that these 2 bits can be leveraged with the next group of 3-bit data, but it will result in irregular memory access patterns, which will degrade the hardware utilization more seriously. Moreover, 8-bit quantization is able to achieve comparable performance with the full precision counterparts for many networks~\cite{Esser2020LEARNED} and the normalized quantization error change is small when the bitwidth is greater than 8 as indicated in Corollary~1 and Fig.~2. Therefore, there is no need to consider a bitwidth higher than 8.%

\section{More details about search space for model compression}
\label{sec:search_space}
Given an uncompressed network with $L$ layers, 
\guo{we use $C_l$ to denote the number of filters at the $l$-th layer.}
\guo{Let $B$ be the number of filters in a group. For any layer $l$, there would be $G = \left \lceil \frac{C_l}{B} \right \rceil$ groups in total.}
\guo{Since we quantize both weights and activations, given $K$ candidate bitwidths, there are $K^2$ different quantization configurations for each layer.}
Thus, 
\guo{for the whole network with $L$ layers}, 
the size of the search space $\Omega$ can be computed by
\begin{equation}
    \label{eq:search_space}
    | \Omega | =  \prod_{l=1}^L \left(K^2 \times G \right).
\end{equation}
Eq.~(\ref{eq:search_space}) indicates that the search space is extremely large when we have large $L$, $K$, and smaller $B$. For example, the search space of ResNet-18 is $|\Omega| \approx 4.3 \times 10^{34}$ with $K=3$ and $B=16$, which is large enough to cover the potentially good configurations.

\begin{table}[!htb]
\renewcommand{\arraystretch}{1.3}
\caption{
\zheng{Comparisons of different methods in terms of MobileNetV3 on CIFAR-100.}
}
\begin{center}
\scalebox{1.0}
{
\begin{tabular}{ccccc}
\toprule
Method & BOPs (M) & BOP comp. ratio & Top-1 Acc. (\%) & Top-5 Acc. (\%)  \\ 
\midrule
Full-precision & 68170.1 & 1.0 & 76.1 & 93.9\\
6-bit precision & 2412.6 & 28.3 & 76.1$\pm$0.0 & 93.7$\pm$0.0\\
\cdashline{1-5}
DQ & 2136.3 & 31.9 & 75.9$\pm$0.1 & 93.7$\pm$0.1   \\
HAQ & 2191.7 & 31.1 & 76.1$\pm$0.1 & 93.5$\pm$0.0\\
DNAS  & 2051.9 & 33.2 & 76.1$\pm$0.1 & 93.7$\pm$0.1   \\
\cdashline{1-5}
\methodshortname-P (Ours) & 59465.8 & 1.1 & 76.0$\pm$0.0 & 93.5$\pm$0.0   \\
\methodshortname-Q (Ours) & 2021.9 & 33.7 & 76.1$\pm$0.1 & 93.7$\pm$0.1 \\
\cellcolor{ZhuangGreen} \methodshortname (Ours) & \cellcolor{ZhuangGreen}\textbf{2006.6} & \cellcolor{ZhuangGreen}\textbf{34.0} & \cellcolor{ZhuangGreen}\textbf{76.1$\pm$0.1} & \cellcolor{ZhuangGreen}\textbf{93.7$\pm$0.1}  \\
\bottomrule
\end{tabular}
}
\end{center}
\label{table:results_on_MobileNetV3}
\vspace{-0.15in}
\end{table}

\section{More details about quantization configurations}
\label{sec:quantization_configurations}
Note that all the compared methods follow the quantization configurations in their papers. Specifically, for DQ~\cite{Uhlich2020Mixed}, we parameterize the fixed-point quantizer using case U3 with $\boldsymbol{\theta} = \left[ d, q_{\mathrm{max}} \right]$. We initialize the weights using a pre-trained model. The initial step size is set to $d = 2^{\lfloor \log_2(\max (| \bW |) / (2^{b-1}-1) \rfloor}$ for weights and $2^{-3}$ for activations. The remaining quantization parameters are set such that  the initial bitwidth is 4-bit. For HAQ~\cite{wang2019haq}, we first truncate the weights and activations into the range of $\left[ -v_w, v_w \right]$ and $\left[ 0, v_x \right]$, respectively. We then perform linear quantization for both weights and activations. To find more proper $v_w$ and $v_x$, we minimize the KL-divergence between the original weight distribution $\bW$ and the quantized one $Q(\bW)$. For DNAS~\cite{wu2018mixed}, we follow DoReFa-Net~\cite{zhou2016dorefa} to quantize weights and follow PACT~\cite{choi2018pact} to quantize activations. We initialize the learnable clipping level to 1. For DJPQ~\cite{wang2020diff} and Bayesian Bits~\cite{van2020bayesian}, we directly get the results from original papers.
For other methods, we use the quantization function introduced in Section~3. The trainable quantization intervals $v_x$ and $v_w$ are initialized to 1.

\section{More results in terms of MobileNetV3 on CIFAR-100}
\label{sec:results_mobilenetv3}
\liu{To investigate the effectiveness of the proposed \methodshortname on the lightweight models, we apply our methods to MobileNetV3 on CIFAR-100. Following LSQ+~\cite{Bhalgat_2020_CVPR_Workshops}, we introduce a learnable offset to handle the negative activations in hard-swish. We show the results in Table~\ref{table:results_on_MobileNetV3}. From the results, our proposed \methodshortname outperforms the compared methods with much fewer computational cost, which demonstrates the effectiveness of the proposed \methodshortname on the lightweight models.}

\section{More details about the compression configurations of the compressed models}
\label{sec:detailed_structure_compressed_network}
In this section, we illustrate the detailed compression configurations (\ie, bitwidth and/or pruning rate) of each layer from the quantized and compressed ResNet-18 and MobileNetV2 in Figures~\ref{fig:lbsq_r18}, ~\ref{fig:lbsq_mbv2}, ~\ref{fig:config_r18}, and ~\ref{fig:config_mbv2}. From the results, our \methodshortname assigns more bitwidths to the weights in the downsampling convolutional layer in ResNet-18 and depthwise convolutional layer in MobileNetV2. Intuitively, one possible reason is that the number of parameters and computational cost of these layers are much smaller than other layers. Compressing these layers may lead to a significant performance decline. Moreover, our \methodshortname inclines to prune more filters in the shallower layers of ResNet-18, which can significantly reduce the number of parameters and computational overhead. To further explore the relationship between pruning and quantization, we illustrate the detailed configurations of each layer from a more compact MobileNetV2, as shown in Fig.~\ref{fig:config_mbv2_compact}. From the figure, if a layer is set to a high pruning rate, our \methodshortname tends to select a higher bitwidth to compensate for the performance decline. In contrast, if a layer is set to a low pruning rate, our \methodshortname tends to select a lower bitwidth to reduce the model size and computational costs.

\begin{figure*}[!htb]
    \centering
    \includegraphics[width=0.9\linewidth]{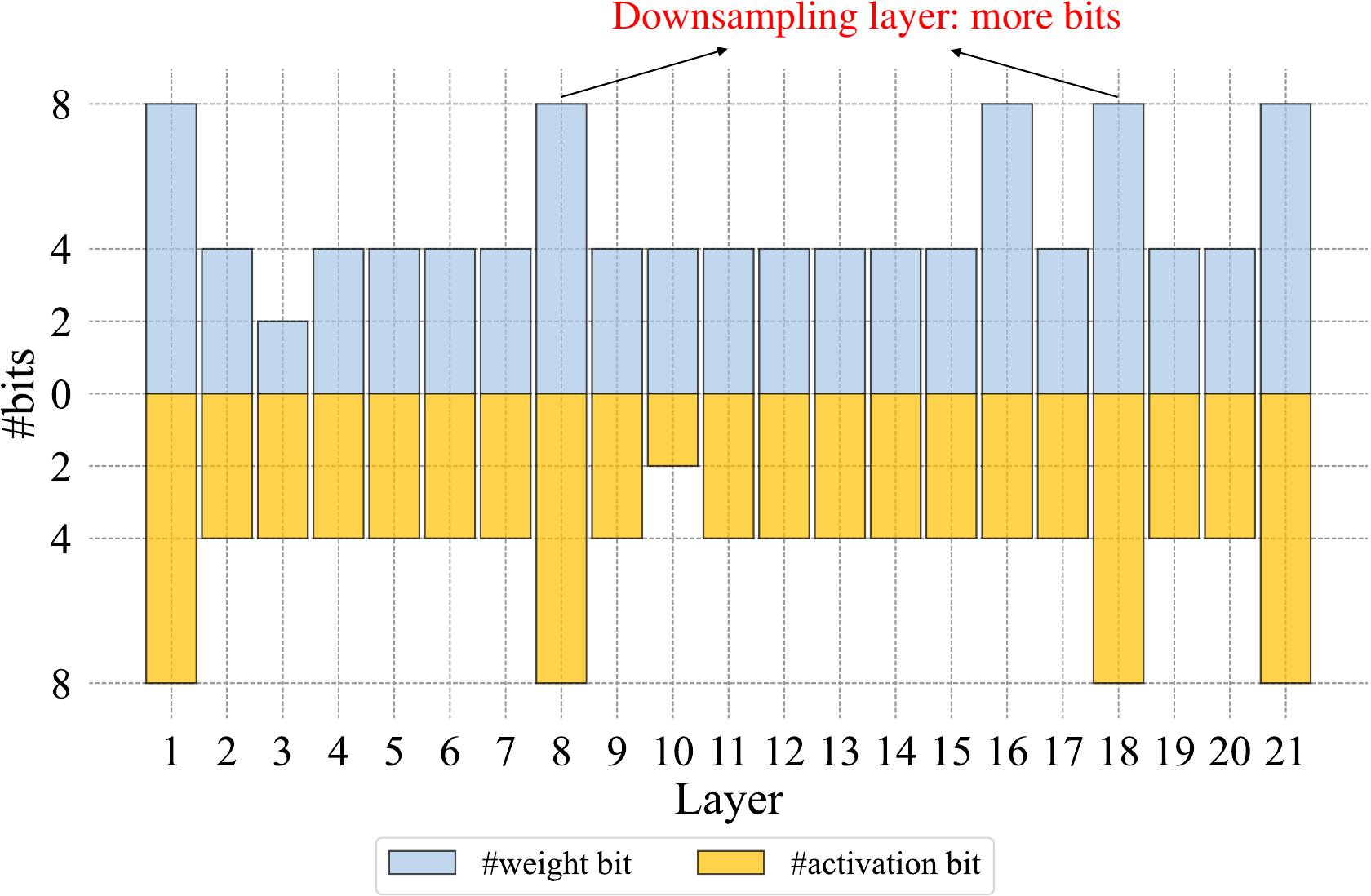}
    \caption{\methodshortname-Q searched bitwidths of ResNet-18 on ImageNet.
    The Top-1 accuracy, Top-5 accuracy and BOPs of the quantized ResNet-18 are 69.5\%, 89.0\% and 34.4G, respectively.
    }
    \label{fig:lbsq_r18}
\end{figure*}

\begin{figure*}[!htb]
    \centering
    \includegraphics[width=0.9\linewidth]{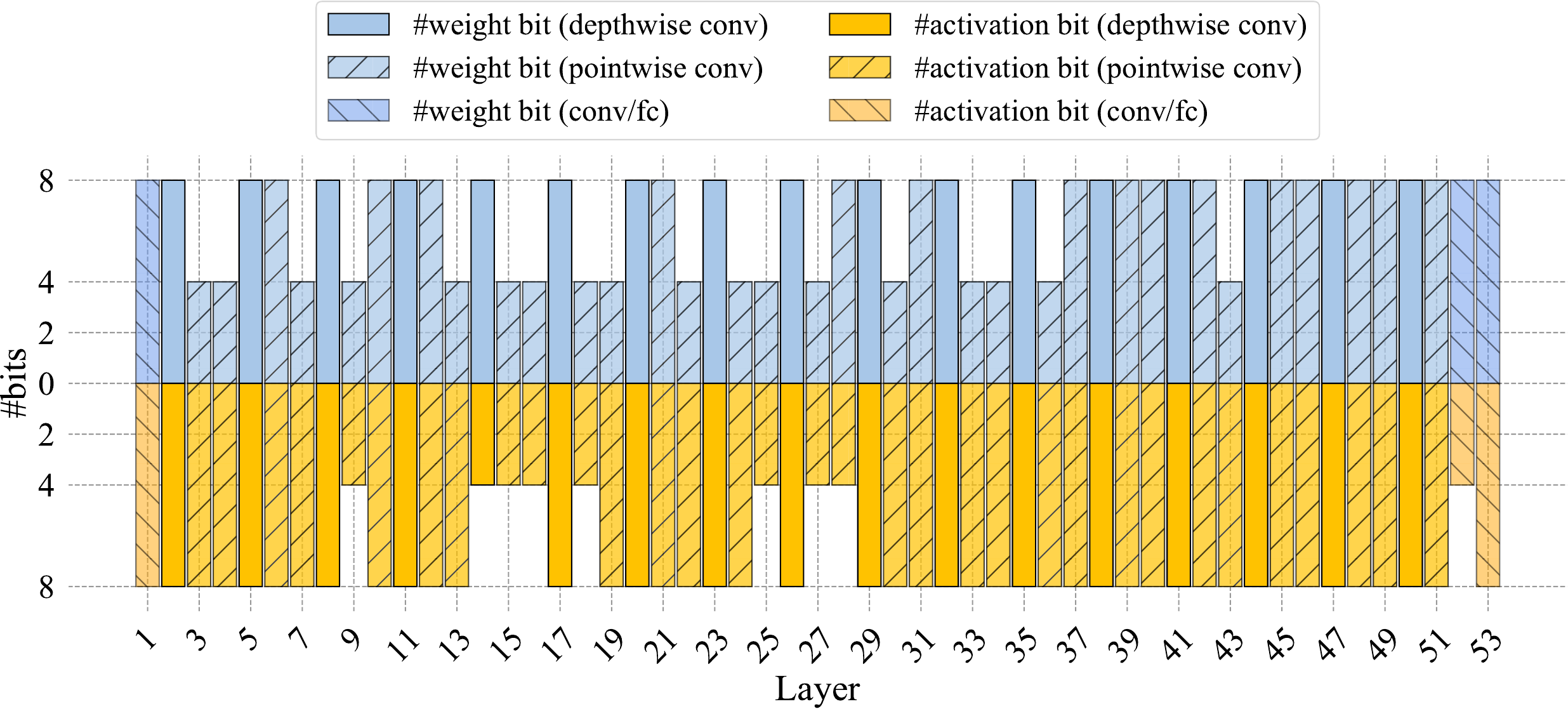}
    \caption{\methodshortname-Q searched bitwidths of MobileNetV2 on ImageNet.
    The Top-1 accuracy, Top-5 accuracy and BOPs of the quantized MobileNetV2 are 71.6\%, 90.2\% and 13.6G, respectively.
    }
    \label{fig:lbsq_mbv2}
\end{figure*}

\begin{figure*}[!htb]
    \centering
    \subfigure[
    \methodshortname searched bitwidths of ResNet-18.
    ]{
        \includegraphics[width=0.9\linewidth]{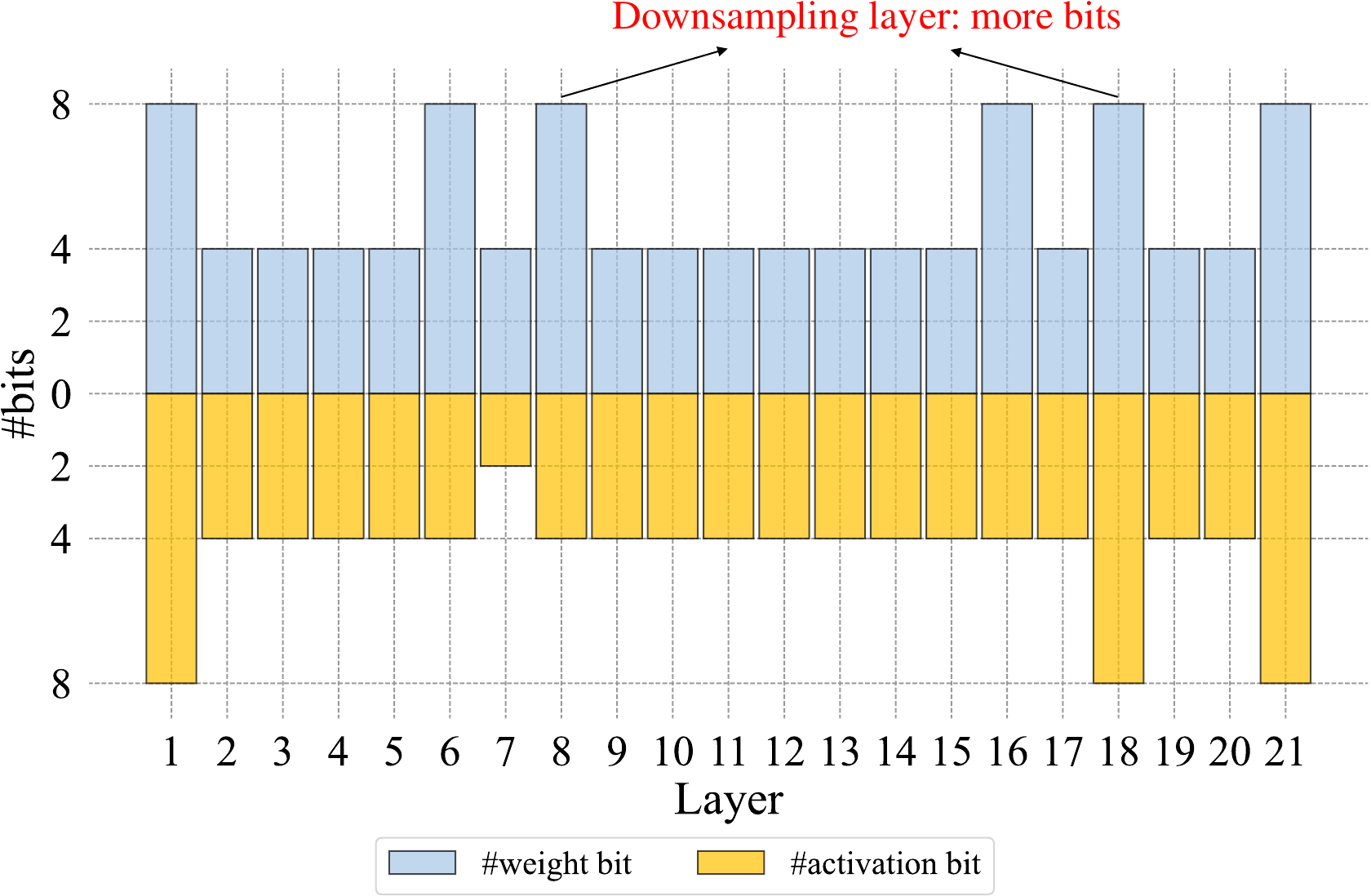}
        \label{fig:bitwidth_r18}
    }
    \hspace{3mm}
    \subfigure[
    \methodshortname searched pruning rates of ResNet-18.
    ]{
        \includegraphics[width=0.9\linewidth]{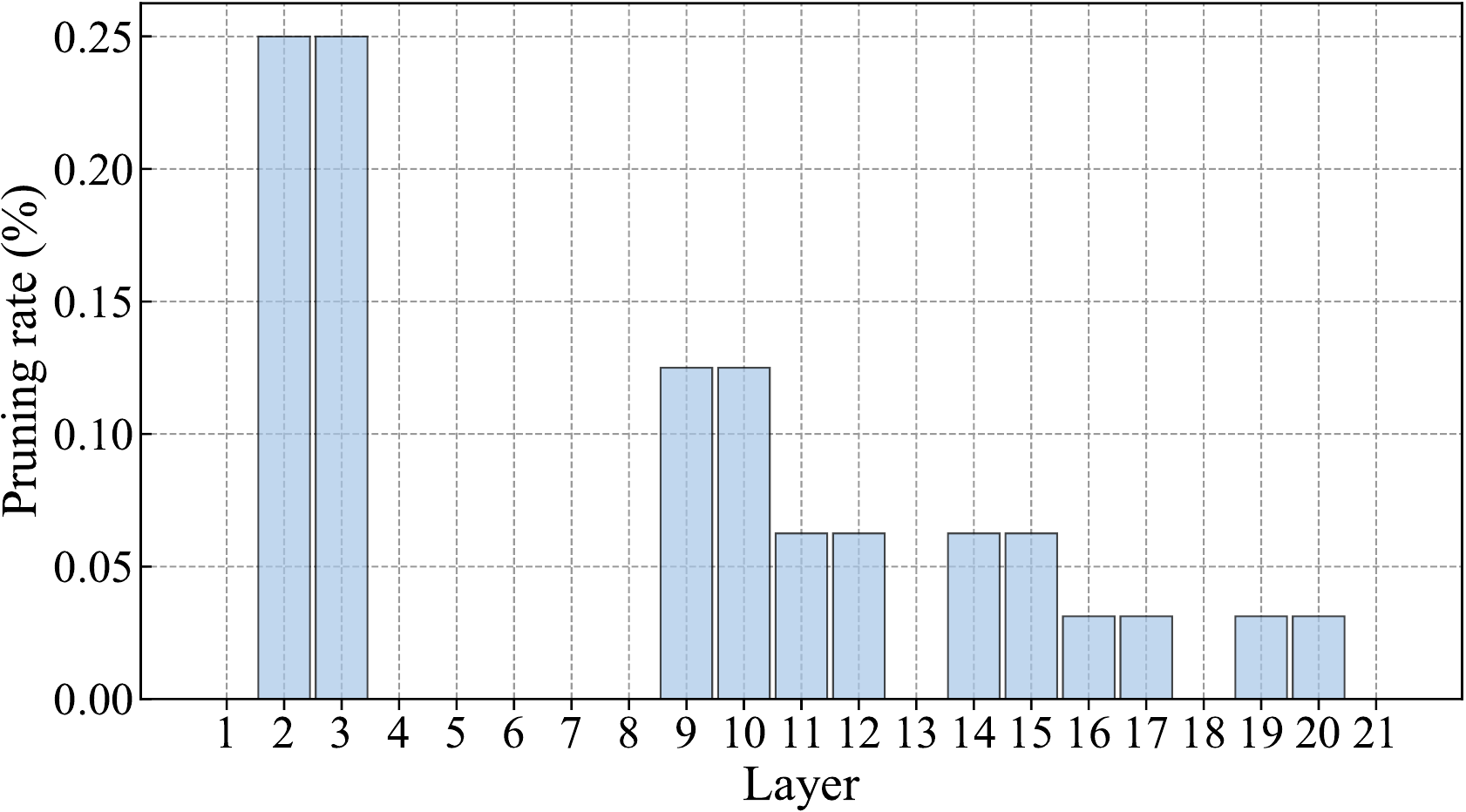}
        \label{fig:pruning_rate_r18}
    }
    \caption{
    \methodshortname searched configurations of ResNet-18 on ImageNet.
    The Top-1 accuracy, Top-5 accuracy and BOPs of the compressed ResNet-18 are 69.6\%, 89.0\% and 34.0G, respectively.
    }
    \label{fig:config_r18}
\end{figure*}

\begin{figure*}[!htb]
    \centering
    \subfigure[
    \methodshortname searched bitwidths of MobileNetV2.
    ]
    {
        \includegraphics[width=1.0\linewidth]{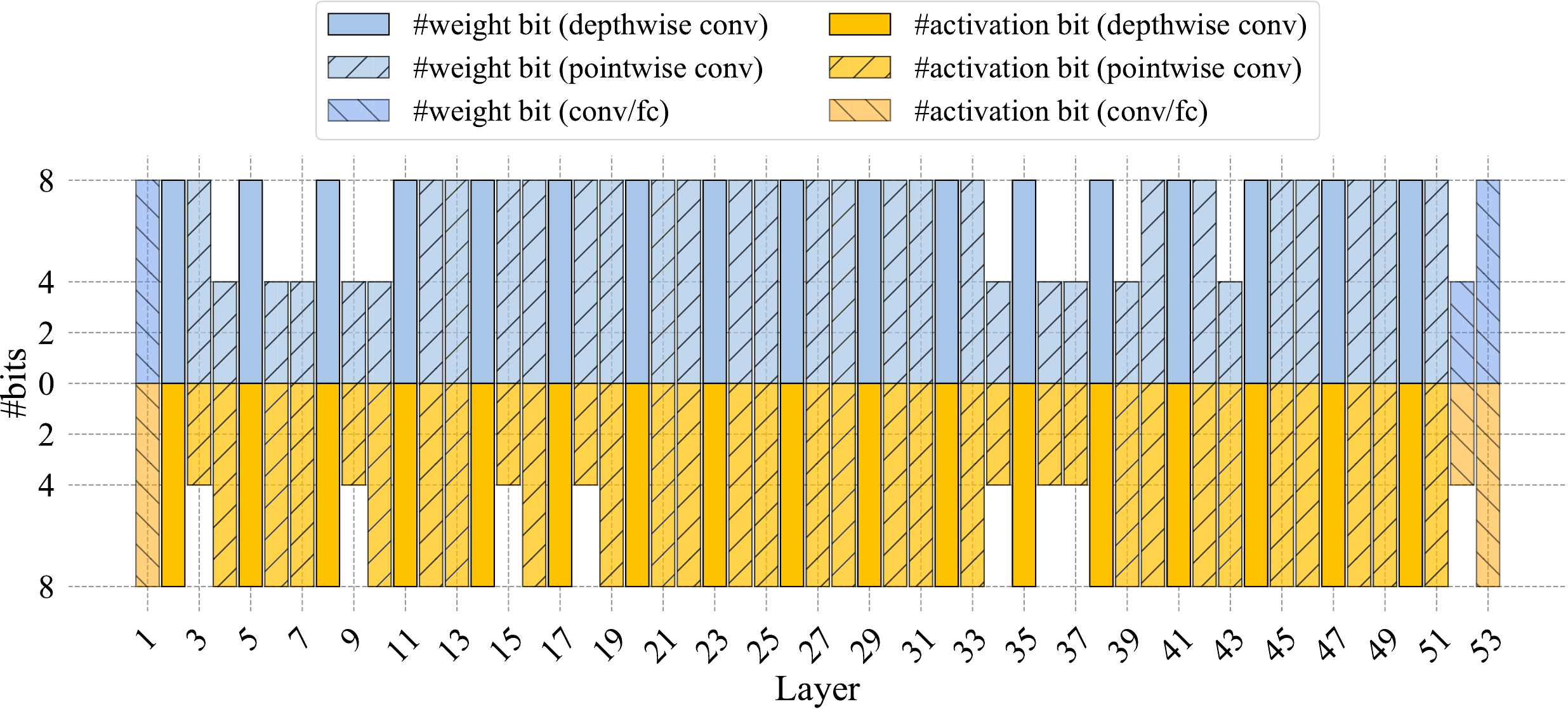}
        \label{fig:bitwidth_mbv2_13554}
    }
    \hspace{3mm}
    \subfigure[
    \methodshortname searched pruning rates of MobileNetV2.
    ]{
        \includegraphics[width=1.0\linewidth]{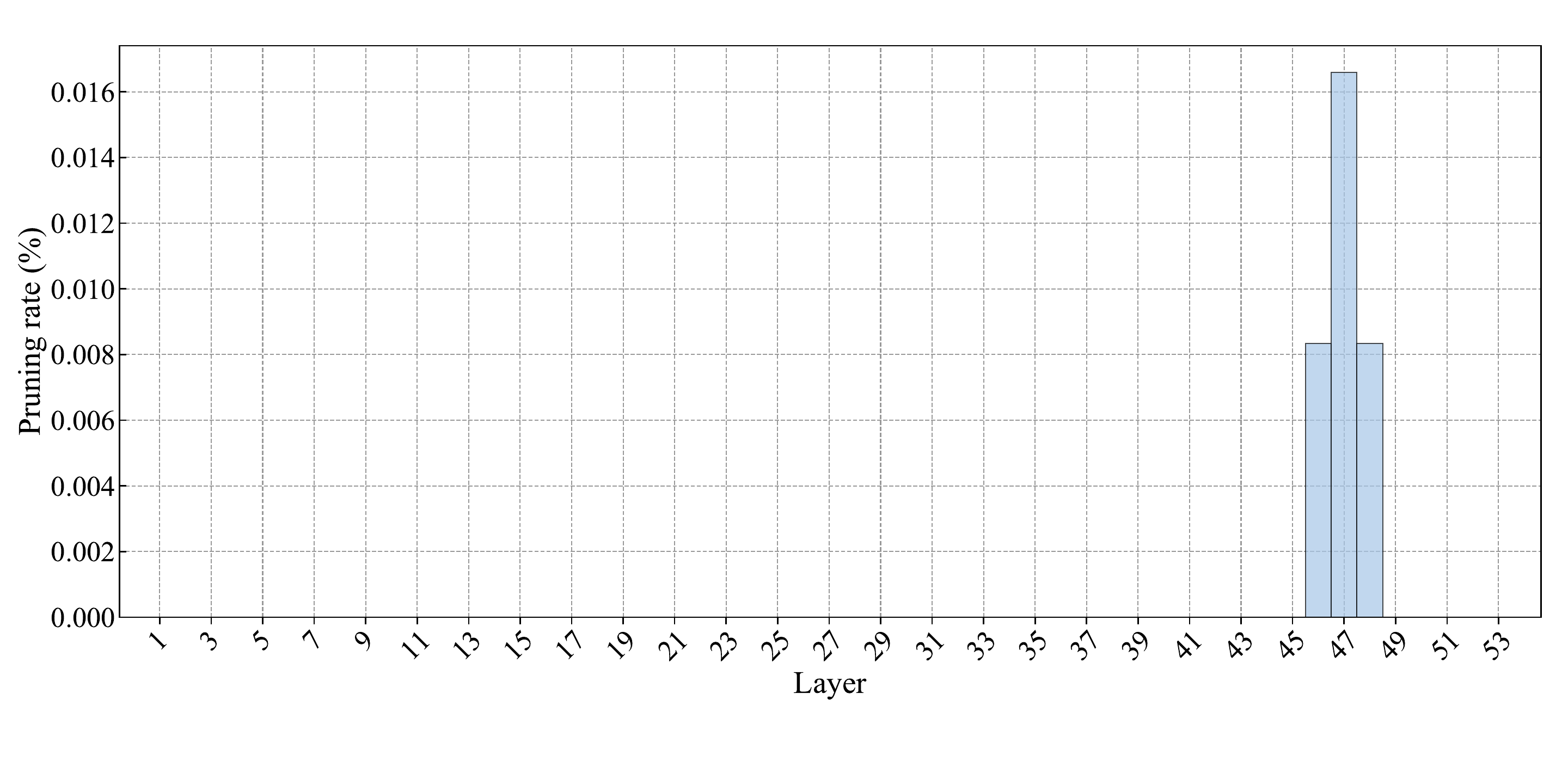}
        \label{fig:pruning_rate_mbv2_13554}
    }
    \caption{\methodshortname searched configurations of MobileNetV2 on ImageNet.
    The Top-1 accuracy, Top-5 accuracy and BOPs of the compressed MobileNetV2 are 71.8\%, 90.3\% and 13.6G, respectively.
    }
    \label{fig:config_mbv2}
\end{figure*}

\begin{figure*}[!htb]
    \centering
    \subfigure[
    \methodshortname searched bitwidths of MobileNetV2.
    ]{
        \includegraphics[width=1.0\linewidth]{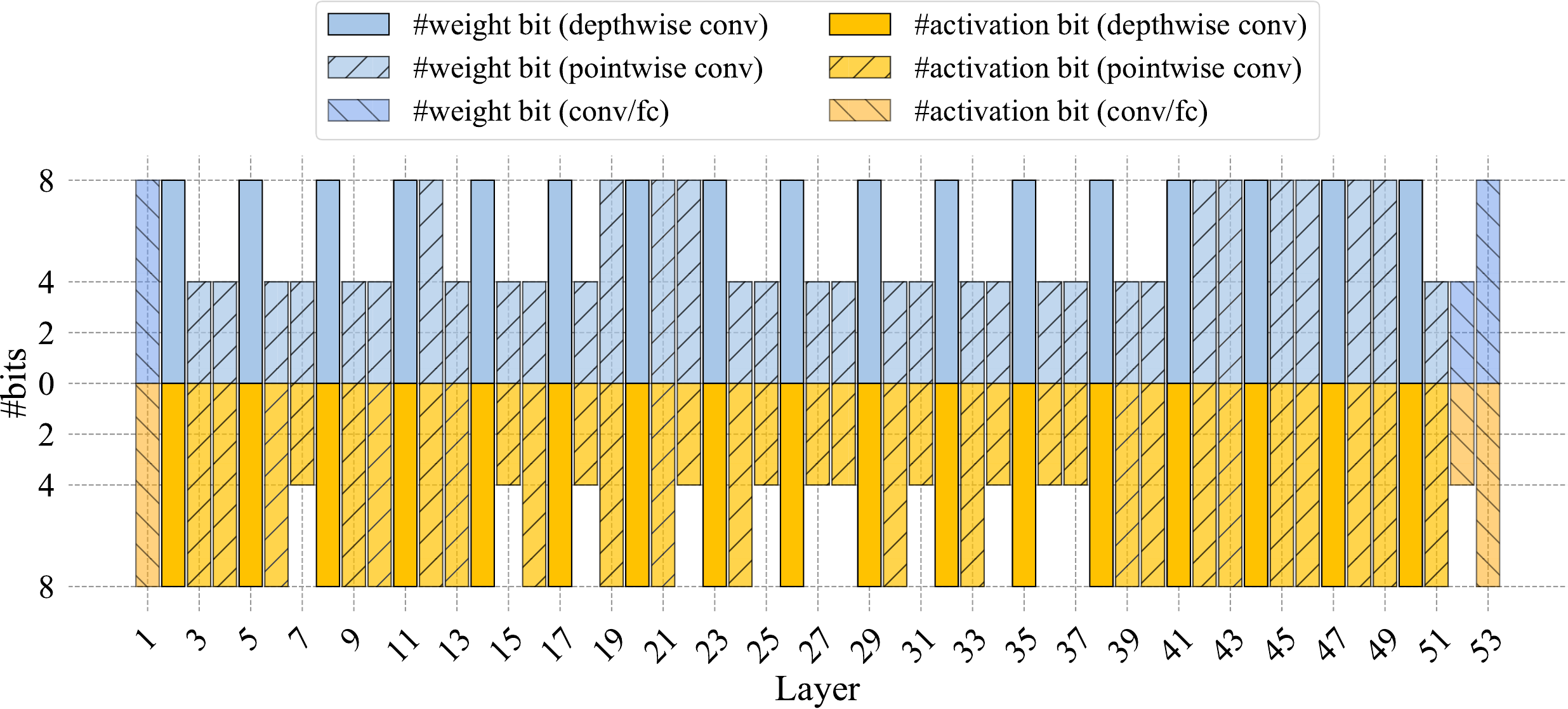}
        \label{fig:bitwidth_mbv2_10835}
    }
    \hspace{3mm}
    \subfigure[
    \methodshortname searched pruning rates of MobileNetV2.
    ]{
        \includegraphics[width=1.0\linewidth]{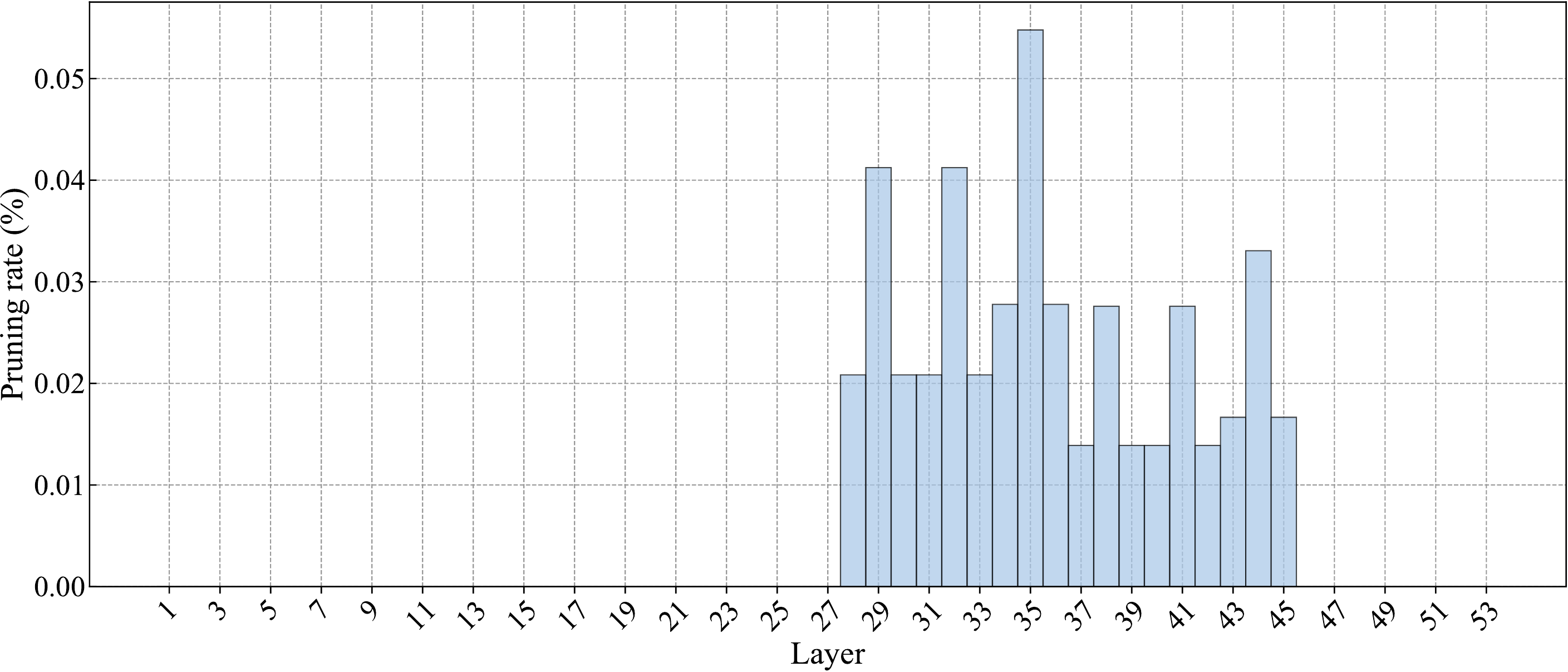}
        \label{fig:pruning_rate_mbv2_10835}
    }
    \caption{
    \methodshortname searched configurations of MobileNetV2 on ImageNet.
    The Top-1 accuracy, Top-5 accuracy and BOPs of the compressed MobileNetV2 are 71.4\%, 90.1\% and 10.8G, respectively.
    }
    \label{fig:config_mbv2_compact}
\end{figure*}

\end{document}